%% file: draft_arxiv.tex
\documentclass{article}

\usepackage{times}
\usepackage{graphicx} 
\usepackage{subfigure} 
\usepackage{natbib}

\usepackage{algorithm}
\usepackage{algpseudocode}

\usepackage{hyperref}
\usepackage[top=0.75in,bottom=0.75in,right =1in,left=1in]{geometry}

\usepackage[utf8]{inputenc}    
\usepackage[T1]{fontenc}       
\usepackage{url}               
\usepackage{booktabs}          
\usepackage{amsfonts}          
\usepackage{natbib}
\usepackage{array}
\usepackage{amsthm}
\usepackage{amsmath}
\usepackage{amssymb}
\usepackage{graphicx}
\usepackage{mathrsfs}
\usepackage{authblk}
\usepackage[toc,page]{appendix}
\input{symbolsdef.tex}
\title{Minimax Regret Bounds for Reinforcement Learning}
\author{Mohammad Gheshlaghi Azar}
\author{Ian Osband}
\author{Remi Munos}
\affil{DeepMind, London, UK}

\begin{document} 

\maketitle

\vskip 0.3in

\begin{abstract} 
We consider the problem of provably optimal exploration in reinforcement learning for finite horizon MDPs.
We show that an optimistic modification to value iteration achieves a regret bound of $\wt{O}( \sqrt{HSAT} + H^2S^2A+H\sqrt{T})$ where $H$ is the time horizon, $S$ the number of states, $A$ the number of actions and $T$ the number of time-steps.
This result improves over the best previous known bound $\wt{O}(HS \sqrt{AT})$ achieved by the UCRL2 algorithm of~\citet{UCRLAuer}. 
The key significance of our new results is that when  $T\geq H^3S^3A$ and $SA\geq H$, it leads to a regret of $\wt{O}(\sqrt{HSAT})$ that matches the established lower bound of $\Omega(\sqrt{HSAT})$ up to a logarithmic factor.
Our analysis contains two key insights.
We use careful application of concentration inequalities to the optimal value function as a whole, rather than to the transitions probabilities (to improve scaling in $S$), and we define Bernstein-based "exploration bonuses" that use the empirical variance of the estimated values at the next states (to improve scaling in $H$).
\end{abstract} 

\section{Introduction}

We consider the reinforcement learning (RL) problem of an agent interacting with an environment in order to maximize its cumulative rewards through time \citep{burnetas1997optimal,Sutton1998}.
We model the environment as a Markov decision process (MDP) whose transition dynamics are unknown from the agent.
As the agent interacts with the environment it observes the states, actions and rewards generated by the system dynamics.
This leads to a fundamental trade off: should the agent explore poorly-understood states and actions  to gain information and improve future performance, or exploit its  knowledge to optimize short-run rewards.

The most common approach to this learning problem is to separate the process of estimation and optimization.
In this paradigm, point estimates of the unknown quantities are used in place of the unknown parameters and a plan is made with respect to these estimates.
Naive optimization with respect to these point estimates can lead to premature exploitation and so may never learn the optimal policy.
Dithering approaches to exploration (e.g., $\epsilon$-greedy) address this failing through random action selection.
However, as this exploration is not directed the resultant algorithms may take exponentially long to learn \citep{Kearns2002}.
In order to learn efficiently it is necessary that the agent prioritizes potentially informative states and actions.
To do this, it is important that the agent maintains some notion of its own uncertainty.
In some sense, given any prior belief, the optimal solution to this exploration/exploitation dilemma is given by the dynamic programming in the extended Bayesian belief state \citep{Berst07a}.
However, the computational demands of this method become intractable for even small problems \citep{guez2013scalable} while finite approximations can be arbitrarily poor \citep{munos2014bandits}.

To combat these failings, the majority of provably efficient learning algorithms employ a heuristic principle known as \textit{optimism in the face of uncertainty} (OFU).
In these algorithms, each state and action is afforded some ``optimism'' such that its imagined value is as high as statistically plausible.
The agent then chooses a policy under this optimistic view of the world.
This allows for efficient exploration since poorly-understood states and actions are afforded higher optimistic bonus.
As the agent resolves its uncertainty, the effects of optimism will reduce and the agent's policy will approach optimality.
Almost all reinforcement learning algorithms with polynomial bounds on sample complexity employ optimism to guide exploration \citep{Kearns2002,Brafman2002,Strehl2006,dann2017}.

An alternative principle motivated by the Thompson sampling  \citep{Thompson1933} has emerged as a practical competitor to optimism.
The algorithm \textit{posterior sampling  reinforcement learning} (PSRL) maintains a posterior distribution for MDPs and, at each episode of interaction, follows a policy which is optimal for a single random sample \citep{Strens00}.

Previous works have argue for the potential benefits of such PSRL methods over existing optimistic approaches \citep{osband2013more,osband2016posterior} but they come with guarantees on the Bayesian regret only.

However a very recent work \cite{agrawal2017} have shown that an optimistic version of posterior sampling (using a max over several samples) achieves a frequentist regret bound $\wt O(H\sqrt{SAT})$ (for large $T$) in the more general setting of weakly communicating MDPs.

In this paper we present a conceptually simple and computationally efficient approach to optimistic reinforcement learning in  finite-horizon  MDPs and report results for the frequentist regret.
Our algorithm, \textit{upper confidence bound value iteration} (UCBVI) is similar to \textit{model-based interval estimation} (MBIE-EB) \citep{strehl2005theoretical}  with a delicate alteration to the form of the ``exploration bonus''. In  particular UCBVI replaces the universal scalar of the bonus  in MBIE-EB with the empirical variance of the next-state value function of each state-action pair. This alteration is essential to improve the regret bound from $\wt O(H)$ to $\wt O(\sqrt{H})$.

Our key contribution is to establish a high probability regret bound $\wt{O}( \sqrt{HSAT}+H^2 S^2 A+H\sqrt{T})$ where $S$ is the number of states, $A$ is the number of actions, $H$ is the episode length and $T$ is the total number of time-steps (and where $\wt{O}$ ignores logarithmic factors).
Importantly, for $T > H^3 S^3 A$ and $SA\geq H$ this bound is $\wt{O}(\sqrt{HSAT})$, which matches the established lower bound for this problem, up to logarithmic factors \citep{osband2016lower}.\footnote{In fact the lower bound of \citep{UCRLAuer} is for the more general setting of the weakly communicating MDPs and it doesn't  directly apply to our setting. But a similar approach can be used to prove a lower bound of same order for the finite-horizon MDPs, as it is already used in \citep{osband2016lower}.}
This positive result is the first of its kind and helps to address an ongoing question about where the fundamental lower bounds lie for reinforcement learning in finite horizon MDPs \citep{Bartlett2009,dann2015sample,osband2016lower}.
Our refined analysis contains two key ingredients:
\begin{itemize}
  \item We use careful application of Bernstein and  Freedman inequalities \citep{bernstein1927theory,freedman1975tail} to the concentration of the \textit{optimal value function} directly, rather than building confidence sets for the transitions probabilities and rewards, like in UCRL2 \citep{UCRLAuer} and UCFH \citep{dann2015sample}.
  \item We use empirical-variance exploration bonuses based on Bernstein's inequality, which together with a recursive Bellman-type Law of Total Variance (LTV) provide tight bounds on the expected sum of the variances of the value estimates, in a similar spirit to the analysis from \cite{azar2013minimax,corr/LattimoreHutter}.
\end{itemize}
At a high level, this work addresses the noted shortcomings of existing RL algorithms \citep{Bartlett2009,UCRLAuer,osband2016posterior}, in terms of dependency on $S$ and $H$. We demonstrates that it is possible to design a simple and computationally efficient optimistic algorithm that  simultaneously address both the loose scaling in $S$ and $H$ to obtain the first regret bounds that match the $\Omega(\sqrt{HSAT})$ lower bounds as $T$ becomes large.

We should be careful to mention the current limitations of our work, each of which may provide fruitful ground for future research.
First, we study the setting of episodic, finite horizon MDPs and not the more general setting of weakly communicating systems \citep{Bartlett2009,UCRLAuer}.
Also we assume that the horizon length $H$ is known to the learner.
Further, our bounds only improve over previous scaling $\wt{O}(HS \sqrt{AT})$ for $T>H^3S^3A$. 

We hope that this work will serve to elucidate several of the existing shortcomings of  exploration in the tabular setting and help further the direction of research towards provably optimal exploration in reinforcement learning.
\section{Problem formulation}
\label{sec: problem formulation}
In this section, we briefly review some notation, as well as  some standard concepts and definitions from the theory of Markov decision processes (MDPs).

\textbf{Markov Decision Problems} We consider the problem of  undiscounted episodic reinforcement learning (RL)~\citep{Tsitsiklis96}, where  an RL agent interacts with a stochastic environment and this interaction is modeled as a discrete-time MDP. An MDP  is a quintuple $\left<\mathcal{S, A }, P , R , H \right>$, where $\mathcal{S}$ and $\mathcal{A}$ are the set of states and actions, $P$ is the state transition distribution, The function $R:\S \times \A \to \Re$  is a real-valued function on the state-action space and the horizon $H$ is the length of episode. We denote by $P(\cdot|x,a)$ and $R(x,a)$ the probability distribution over the next state and the immediate reward of taking action $a$ at state $x$, respectively. The agent interacts with the environment in a sequence of episodes. The interaction between the agent and environment at every episode\footnote{We write $[n]$ for $\{ i \in \Nat \mid 1 \le i \le n\}$.
} $k\in[K]$ is as follows: starting from $x_{k,1}\in \S$ which is chosen by the environment, the agent interacts with the environment for $H$ steps by following a sequence of actions chosen in $\A$ and observes a sequence of next-states and rewards until the end of episode. The initial state $x_{k,1}$ may change arbitrarily from one episode to the next. We also use the notation $\|\cdot\|_1$ for the $\ell_1$ norm throughout this paper.

\begin{assumption}[MDP Regularity]
\label{asm:regular}
We assume   $\mathcal{S}$ and $\mathcal{A}$ are finite sets with cardinalities  $S$, $A$, respectively. We  also assume that the immediate reward $R(x,a)$ is deterministic and belongs to the interval $[0,1]$.\footnote{For rewards in $[R_{\min},R_{\max}]$ simply rescale these bounds.}
\end{assumption}
In this paper we focus on the setting where the reward function $R$ is known, but extending our algorithm to unknown stochastic rewards poses no real difficulty.

The policy during an episode is expressed as a mapping $\pi:\S\times [H]\to \A$.  The {\em value}
$V^{\pi}_h:\S\to \mathbb R$ denotes the value function at every step  $h=1,2,\dots,H$ and state $x\in\S$ such that $V^{\pi}_h(x)$ corresponds to the expected sum of $H-h$ rewards received under policy $\pi$, starting from $x_h=x\in \S$.  Under Assumption  \ref{asm:regular} there exists  always a  policy $\pi^*$ which attains the best possible values, and we define the {\em optimal value function} $V^*_h(x)\eqdef  \sup_{\pi}V^{\pi}_h(x)$ for all  $x\in \S$  and  $h\geq 1$. The policy  $\pi$ at every  step $h$ defines the state transition kernel $P^{\pi}_h$  and the reward function $r^{\pi}_h$ as $P^{\pi}_h(y|x)\eqdef   P(y|x,\pi(x,h))$  and $r^{\pi}_h(x)\eqdef   R(x,\pi(x,h))$ for all $\inset{x}{\S}$. For every $V:\S\to\mathbb R$ the right-linear operators 
 $P\cdot$ and $P^{\pi}_h\cdot$ are 
 also defined as $(PV)(x,a)\eqdef  \sum{_{\inset{y}{\S}}}P(y|x,a){V}(y)$ for all $\inset{(x,a)}{ \S\times \A}$ and $(P^{\pi}_hV)(x)\eqdef  \sum_{\inset{y}{\S}}P^{\pi}_h(y|s){V}(y)$ for all $\inset{x}{\S}$, respectively. The Bellman operator for the  policy $\pi$, at every step $h>0$ and $\inset{x}{S}$,  is defined as $(\T^{\pi}_h V)(x)\eqdef  r^{\pi}_h(x)+  (P^{\pi}_h V)(x)$. We also define the state-action Bellman operator for all $\inset{(x,a)}{\S\times \A}$ as $(\mathcal{T}V)(x,a)\eqdef  R(x,a)+  (PV)(x,a)$
 and the optimality Bellman operator for all $x\in\S$ as $(\mathcal{T}^* V)(x)\eqdef  \max_{a\in\A}(\mathcal{T}V)(x,a)$. For ease of exposition, we remove the dependence on  $x$ and $(x,a)$, e.g., writing $PV$ for $(PV)(x,a)$ and $V$ for $V(x)$, when there is no possible confusion.

We measure the performance of the learner over $T=KH$ steps\footnote{In this paper we will often substitute
{\medmuskip=1mu
\thinmuskip=1mu
\thickmuskip=1mu $T=KH$} to highlight various dependencies relative to the existing literature.
This equivalence should be kept in mind by the reader.}
by the regret $\mathrm{Regret}(K)$, defined as
\begin{equation*}\label{eq:regret}
\mathrm{Regret}(K) \eqdef  \sum_{k=1}^{K} V^*_1(x_{k,1}) - V^{\pi_k}_1(x_{k,1}),
\end{equation*}
where $\pi_k$ is the control policy followed by the learner at episode $k$. Thus the regret measures the expected loss of following the policy produced by the learner instead of the optimal policy. So the goal of learner is to follow  a sequence of policies $\pi_1,\pi_2,\dots,\pi_K$  such that $\text{Regret}(K)$ is as small as possible.

\section{Upper confidence bound value iteration}
\label{sec: alg}

In this section we introduce two variants of the algorithm that we investigate in this paper.
We call the algorithm \textit{upper confidence bound value iteration} (UCBVI).
UCBVI is an extension of value iteration which guarantees that the resultant value function is a (high-probability) upper confidence bound (UCB) on the optimal value function.
This algorithm is related to the \textit{model based interval estimation} (MBIE-EB) algorithm \citep{strehl2008analysis}.
Our key contribution is the precise design of the upper confidence sets, and the analysis which lead to tight regret bounds.

\texttt{UCBVI}, described in Algorithm \ref{alg:ofu_rl}, calls \texttt{UCB-Q-values} (Algorithm \ref{alg:ucbvi}) which returns UCBs on the Q-values computed by value iteration using an empirical Bellman operator to which is added a confidence bonus \texttt{bonus}. We consider two variants of UCBVI depending on the structure of \texttt{bonus}, which we present in Algorithms \ref{alg:bonus_1} and \ref{alg:bonus_2}.

\begin{algorithm}[H]
\caption{\texttt{UCBVI}}
\label{alg:ofu_rl}
\begin{algorithmic}
\State Initialize data $\Hc = \emptyset$ 
\For{episode $k=1,2,\dots,K$}
\State $Q_{k,h} = \mathtt{UCB-Q-values}(\Hc)$
\For{step $h=1,\dots,H$}
\State Take action $a_{k,h} = \arg\max_a Q_{k,h}(x_{k,h}, a)$
\State Update $\Hc = \Hc \cup (x_{k,h}, a_{k,h}, x_{k,h+1})$
\EndFor
\EndFor
\end{algorithmic}
\end{algorithm}

\begin{center}
\begin{algorithm}[H]
\caption{\texttt{UCB-Q-values}}
\label{alg:ucbvi}
\begin{algorithmic}
\medmuskip=2mu
\thinmuskip=1mu
\thickmuskip=2mu

\Require Bonus algorithm \texttt{bonus}, Data $\Hc$
\State Compute,  for all $(x,a,y) \in \Sc  \times \Ac \times \Sc$,
\State $N_k(x,a,y) = \sum_{(x',a',y') \in \Hc} \Ind(x'=x, a'=a, y'=y)$
\State $N_k(x,a) = \sum_{y \in \Sc} N_k(x, a, y) $
\State $N'_{k,h}(x,a) = \sum_{(x_{i,h},a_{i,h},x_{i,h+1}) \in \Hc} \Ind(x_{i,h}=x, a_{i,h}=a)$
\State Let $\Kc = \{(x,a) \in \Sc \times \Ac, \; N_k(x,a) > 0\}$ 
\State Estimate $\wh{P}_{k}(y|x,a) = \frac{N_k(x,a,y)}{N_k(x,a)}$ for all $(x,a) \in \Kc$ 
\State Initialize $V_{k,H+1}(x) = 0$ for all $(x,a) \in \Sc \times \Ac$ 
\For{$h = H, H-1, \dots , 1$}
\For{$(x,a) \in \Sc \times \Ac$}
  \If{$(x,a) \in \Kc$}
  \State $b_{k,h}(x,a) = \mathtt{bonus}(\wh{P}_k, V_{k, h+1}, N_k, N'_{k,h})$
  \State $Q_{k,h}(x,a) = {\rm min} \big( Q_{k-1,h}(x,a), H,$
  \State \hspace{6mm}$R(x,a) + (\wh{P}_{k} V_{k,h+1})(x,a)+ b_{k,h}(x,a) \big)$
  \Else
  \State $Q_{k,h}(x,a) = H$
\EndIf
\State $V_{k,h}(x)=\max_{a\in \A}Q_{k,h}(x,a)$
\EndFor
\EndFor
\State \Return Q-values $Q_{k,h}$
\end{algorithmic}
\end{algorithm}
\end{center}

The first of these \uCH  is based upon Chernoff-Hoeffding's  concentration inequality,   considers $\mathtt{UCBVI}$ with $\mathtt{bonus} = \mathtt{bonus\_1}$.  $\mathtt{bonus\_1}$ is a very simple bound  which only assumes that values are bounded in $[0,H]$.
We will see in Theorem \ref{thm: ucbvi_1 regret} that this very simple algorithm can already achieve a regret bound of $\wt O(H\sqrt{SAT})$, thus improving the best previously known regret bounds from a $S$ to a $\sqrt{S}$ dependence. The intuition for this improved $S$-dependence  is that our algorithm (as well as our analysis) does not consider confidence sets on the transition dynamics $P(y|x,a)$ like UCRL2 and UCFH do, but instead directly maintains confidence intervals on the optimal value function. This is crucial as, for any given $(x,a)$, the transition dynamics are $S$-dimensional whereas the Q-value function is one-dimensional.

\begin{center}
\begin{algorithm}
\caption{\texttt{bonus\_1}}
\label{alg:bonus_1}
\begin{algorithmic}
\Require{$\wh{P}_k(x,a), N_k(x,a)$}
\State  $b(x,a) = 7 HL \sqrt{\frac{1}{N_k(x,a)}}$ where $L=\ln(5 SAT/\delta),$
\State \Return $b$
\end{algorithmic}
\end{algorithm}
\end{center}

However, the loose form of UCB given by \uCH does not look at the value function of the next state, and just consider it as being bounded in $[0,H]$. However, much better bounds can be obtained by looking at the variance of the next state values. Our main result relies upon $\mathtt{UCBVI}$ with $\mathtt{bonus} = \mathtt{bonus\_2}$, which we refer to as \uBF as it relies on Bernstein-Freedman's concentration inequalities to build the confidence set. \uBF builds upon the intuition for \uCH but also incorporates a variance-dependent exploration bonus. This leads to tighter exploration bonuses and an improved regret bound of $\wt O(\sqrt{HSAT})$. 

\begin{center}
\begin{algorithm}
\caption{\texttt{bonus\_2}}
\label{alg:bonus_2}
\begin{algorithmic}
\Require{$\wh{P}_k(x,a), V_{k, h+1}, N_k, N'_{k,h}$}
\medmuskip=0mu
\thinmuskip=0mu
\thickmuskip=0mu
\State
\begin{align*}
b(x,a)& = \sqrt{\frac{8 L \textrm{Var}_{Y \sim \wh{P}_k(\cdot|x,a)} \left( V_{k,h+1}(Y) \right) }{N_k(x,a)}} 
+
\frac{14 HL} {3N_{k}(x,a)} 
+
\sqrt{\frac{ 8\sum_{y} \wh P_k(y|x,a) \left[ \min\left(\frac{100^2H^3S^2AL^2}{N'_{k,h+1}(y)},H^2\right) \right] }{N_k(x,a)}}
\end{align*}
\State where $L=\ln(5 SAT/\delta)$
\State \Return{$b$}
\end{algorithmic}
\end{algorithm}
\end{center}

Compared to \uBF here we use a bonus built from the empirical variance of the estimated next values. The idea is that if we had knowledge of the optimal value $V^*$, we could build tight confidence bounds using the variance of the optimal value function at the next state in place of the loose bound of $H$. Since however $V^*$ is unknown, here we use as a surrogate the empirical variance of the estimated values.  As more data is gathered, this variance estimate will converge to the variance of $V^*$. Now we need to make sure our estimates $V_{k,h}$ are optimistic (i.e., that they upper bound $V^*_h$) at all times. This is achieved by adding an additional bonus (last term in $b(x,a)$), which guarantees that we upper bound the variance of $V^*$. Now, using an iterative -Bellman-type- Law of Total Variance, we have (see proof) that the sum of the next-state variances of $V^*$ (over $H$ time steps) (which is related to the sum of the exploration bonuses over $H$ steps) is bounded by the variance of the $H$-steps return. Thus the size of the bonuses built by \uBF are constrained over the $H$ steps. And we prove that the sum of those bonuses do not grow linearly in $H$ but in $\sqrt{H}$ only. This is the key for our improved dependence from $H$ to $\sqrt{H}$.

\section{Main results}
\label{sec: main}

In this section we present the main results of the paper, which are upper bounds on the regret of \uCH and \uBF algorithms. We assume Assumption \ref{asm:regular} holds.

\begin{theorem}[Regret bound for \uCH]
\label{thm: ucbvi_1 regret}
\hspace{0.0000001mm} \newline
Consider a parameter $\delta > 0$. Then the regret of \uCH is bounded w.p.~at least $1 - \delta$, by
\beqan
\mathrm{Regret}(K) &\leq& 20   H^{3/2}L\sqrt{S A K} + 250H^2 S^2 A L^2 ,
\eeqan
where $L=\ln(5HSAT/\delta)$.
\end{theorem}

For $T\geq HS^3A$  and $SA\geq H$ this bound translates to a  regret bound of $\wt O(H\sqrt{SAT})$, where $T=KH$ is the total number of time-steps at the end of episode $K$.

Theorem \ref{thm: ucbvi_1 regret} is significant in that, for large $T$, it improves the regret dependence from $S$ to $\sqrt{S}$, compared to the best known bound of \cite{UCRLAuer}. The main intuition for this improved $S$-dependence is that we bound the estimation error of the next-state value function directly, instead of the transition probabilities.

More precisely, instead of bounding the estimation error $(\wh P^{\pi_k}_{k} -P^{\pi_k})V_{k,h+1}$ by $\| \wh P^{\pi_k}_{k} -P^{\pi_k} \|_1 \|V_{k,h+1}\|_{\infty}$ (as is done in \cite{UCRLAuer} for example), we bound $(\wh P^{\pi_k}_{k} -P^{\pi_k})  V^{*}_{h+1}$ instead (for which a bound with no dependence on $S$ can be achieved since $V^*$ is deterministic) and handle carefully the correction term $(\wh P^{\pi_k}_{k} -P^{\pi_k})(V_{k,h+1}- V^{*}_{h+1})$.

Our second result, Theorem \ref{thm: ucbvi_2 regret}, demonstrates that we can improve upon the $H$-dependence by using a more refined, Bernstein-Friedman-type, exploration bonus.

\begin{theorem}[Regret bound for \uBF]
\label{thm: ucbvi_2 regret}
\hspace{0.0000001mm} \newline
Consider a parameter $\delta > 0$. Then the regret of \uBF is bounded w.p.~$1 - \delta$, by
\beqan
\mathrm{Regret}(K) \leq  30 HL \sqrt{S A K} + 2500H^2 S^2 A L^2\notag+ 4H^{3/2}\sqrt{KL},
\eeqan
where $L=\ln(5HSAT/\delta)$.
\end{theorem}

We note that for $T\geq H^3S^3A$  and $SA\geq H$ this bound translates to a  regret bound of $\wt O(\sqrt{HSAT})$. This result is particularly significant since, for $T$ large enough (i.e., $T\geq H^3 S^3 A$), our bound is $\wt O(\sqrt{HSAT})$ which matches the established lower bound $\Omega(\sqrt{HSAT})$ of \citep{UCRLAuer,osband2016lower} up to logarithmic factors.

The key insight is to apply concentration inequalities to bound the estimation errors and the exploration bonuses in terms of the variance of $V^*$ at the next state.
We then use the fact that the sum of these variances is bounded by the variance of the return \citep[see e.g., ][]{Munos99cdc,azar2013minimax,corr/LattimoreHutter}, which shows that the estimation errors accumulate as $\sqrt{H}$ instead of linearly in $H$, thus implying the improved $H$-dependence.

\subsection*{Computational efficiency}

Theorems \ref{thm: ucbvi_1 regret} and \ref{thm: ucbvi_2 regret} guarantee the statistical efficiency of UCBVI.
In addition, both \uCH and \uBF are computationally tractable.
Each episode both algorithms perform an optimistic value iteration with computational cost of the same order as solving a known MDP.
In fact, the computational cost of these algorithms can be further reduced by only selectively recomputing UCBVI after sufficiently many observations.
This technique is common to the literature \cite{UCRLAuer,dann2015sample} and would not affect the $\tilde{O}$ statistical efficiency.
The computational cost of this variant of UCBVI  then amounts to  $\wt O(SA\min(SA,T)\min(T,S))$ as it only needs to update the model $\wt O(SA)$ times \citep{UCRLAuer}.

\subsection*{Weakly communicating MDPs}

In this short paper we focus on the setting of finite horizon MDPs. By comparison, previous optimistic approaches to exploration, such as UCRL2, provide bounds for the more general setting of weakly communicating MDPs \citep{UCRLAuer,Bartlett2009}.

However, we believe that much of the insight from the UCBVI algorithm (and its analysis) will carry over to this more general setting using existing techniques such as `the doubling trick` \citep{UCRLAuer}.

\section{Proof sketch}

Here we provide the sketch proof of our results. The full proof is deferred to the appendix. 
\subsection{Sketch Proof of Theorem 1}

Let $\Omega=\{V_{k,h}\geq V^*_h, \forall k,h\}$ be the event under which all computed $V_{k,h}$ values are upper bounds on the optimal value function. Using backward induction on $h$ (and standard concentration inequalities) one can prove that $\Omega$ holds with high probability (see Lem.~\ref{lem:ucb.basic} in the appendix). To simplify notations in this sketch of proof we will not make the numerical constants explicit, and instead we will denote by $\square$ a numerical constant which can vary from line to line. The exact values of these constants are provided in the full proof. We will also make use of simplified notations, such as using $L$ to represent the logarithmic term $L= \ln(\square HSAT/\delta)$.

The cumulative regret at episode  $K$ is ${\text{Regret}}(K) \eqdef \sum_{1\leq k\leq K} V^*_1(x_{k,1}) - V^{\pi_k}_1(x_{k,1})$. Define  $\wt {\text{Regret}}(K) \eqdef \sum_{1\leq k\leq K}  V_{k,1}(x_{k,1}) - V^{\pi_k}_1(x_{k,1})$. Under $\Omega$ we have ${\text{Regret}}(K)\leq \wt {\text{Regret}}(K)$, so we now bound $\wt {\text{Regret}}(K)$. Define $\Delta_{k,h} \eqdef V^*_h - V^{\pi_k}_h$ and $\wt \Delta_{k,h} \eqdef V_{k,h} - V^{\pi_k}_h$. Thus 
\beqan 
\Delta_{k,h} \leq \wt \Delta_{k,h}= \wh P^{\pi_k}_{k} V_{k,h+1} + b_{k,h} -  P^{\pi_k} V^{\pi_k}_{h+1} 
= (\wh P^{\pi_k}_{k} -P^{\pi_k})  V_{k,h+1}  + P^{\pi_k}  \wt \Delta_{k,h+1} + b_{k,h}.
\eeqan 
The difficulty in bounding $(\wh P^{\pi_k}_{k} -P^{\pi_k})V_{k,h+1}$ is that both $V_{k,h+1}$ and $\wh P^{\pi_k}_{k}$ are random variables and are not independent (the value function $V_{k,h+1}$ computed at $h+1$ may depend on the samples collected from state $x_{h,k}$), thus a straightforward application of Chernoff-Hoeffding (CH) inequality does not work here. In \cite{UCRLAuer}, this issue is addressed by bounding it by $\| \wh P^{\pi_k}_{k} -P^{\pi_k} \|_1 \|V_{k,h+1}\|_{\infty}$ at the price of an additional $\sqrt{S}$. 

The main contribution of our $\wt O(H\sqrt{SAT})$ bound (which removes a $\sqrt{S}$ factor compared to the previous bound of \cite{UCRLAuer}) is to handle this term more properly. Instead of directly bounding $(\wh P^{\pi_k}_{k} -P^{\pi_k})V_{k,h+1}$, we bound $(\wh P^{\pi_k}_{k} -P^{\pi_k})  V^{*}_{h+1}$, using straightforward application of CH (which removes the $\sqrt{S}$ factor since $V^*_{h+1}$ is deterministic), and deal with the correction term $(\wh P^{\pi_k}_{k} -P^{\pi_k}) ( V_{k,h+1} - V^{*}_{h+1})$. We have
\beqan 
\wt \Delta_{k,h} = (\wh P^{\pi_k}_{k} -P^{\pi_k}) ( V_{k,h+1} - V^{*}_{h+1})
+ P^{\pi_k}  \wt \Delta_{k,h+1} +  b_{k,h} + e_{k,h},
\eeqan 
where $e_{k,h} \eqdef (\wh P^{\pi_k}_{k} -P^{\pi_k})  V^{*}_{h+1}(x_{k,h})$ is the estimation error of the optimal value function at the next state. Defining $\wt \delta_{k,h} \eqdef \wt \Delta_{k,h}(x_{k,h})$, we have
\beqan 
\wt \delta_{k,h} \leq
(\wh P^{\pi_k}_{k} -P^{\pi_k}) \Delta_{k,h+1}(x_{k,h})  + \wt \delta_{k,h+1} + \epsilon_{k,h}  + b_{k,h} + e_{k,h},
\eeqan 
where $\epsilon_{k,h} \eqdef P^{\pi_k}\Delta_{k,h+1}(x_{k,h}) - \Delta_{k,h+1}(x_{k,h+1})$.

\paragraph{Step 1: bound on the correction term $(\wh P^{\pi_k}_{k} -P^{\pi_k}) \Delta_{k,h+1}(x_{k,h})$.} Using  Bernstein's inequality (B), this term is bounded by
\beqan 
\sum_y P^{\pi_k}(y|x_{k,h})\sqrt{\frac{\square L}{P^{\pi_k}(y|x_{k,h}) n_{k,h}}} \Delta_{k,h+1}(y) +\frac{\square SH L}{n_{k,h}},
\eeqan 
where $n_{k,h}\eqdef N_k(x_{k,h},\pi_k(x_{k,h}))$. Now considering only the $y$ such that $P^{\pi_k}(y|x_{k,h}) n_{k,h} \geq \square H^2L$, and since $0\leq \Delta_{k,h+1}\leq \wt \Delta_{k,h+1}$, then $(\wh P^{\pi_k}_{k} -P^{\pi_k}) \Delta_{k,h+1}(x_{k,h})$ is bounded by
\beqan 
 \bar \epsilon_{k,h} + \sqrt{\frac{\square L}{P^{\pi_k}(x_{k,h+1}|x_{k,h}) n_{k,h}}} \wt \delta_{k,h+1} + \frac{\square SH L}{n_{k,h}}
 \leq \bar \epsilon_{k,h} + \frac{1}{H} \wt \delta_{k,h+1} + \frac{\square SH L}{n_{k,h}},
\eeqan 
where $\bar \epsilon_{k,h}\eqdef \sqrt{\frac{\square L}{n_{k,h}}} \Big( \sum_y P^{\pi_k}(y|x_{k,h})\frac{\wt \Delta_{k,h+1}(y)}{\sqrt{P^{\pi_k}(y|x_{k,h})}} -  \frac{\wt \delta_{k,h+1}}{\sqrt{P^{\pi_k}(x_{k,h+1}|x_{k,h})}}\Big)$.

The sum over the neglected $y$ such that $P^{\pi_k}(y|x_{k,h}) n_{k,h} < \square H^2L$ contributes to an additional term 
\beqan
\sum_y \sqrt{\frac{\square P^{\pi_k}(y|x_{k,h}) n_{k,h} L}{n_{k,h}^2}} \Delta_{k,h+1}(y) \leq \frac{\square SH^2L}{n_{k,h}}.
\eeqan
Neglecting this term (and the smaller order term $\square SH L/ n_{k,h}$) for now (by the pigeon-hole principle we can prove that these terms contribute to the final regret by a constant at most $\square S^2AH^2L^2$), we have
\beqan 
\wt \delta_{k,h} \leq
\Big( 1+\frac{1}{H}\Big) \wt \delta_{k,h+1} + \epsilon_{k,h} + \bar \epsilon_{k,h}  + b_{k,h}+ e_{k,h} 
\leq
\underbrace{\Big( 1+\frac{1}{H}\Big)^H}_{\leq e}  \sum_{i=h}^{H-1}\Big( \epsilon_{k,i} + \bar \epsilon_{k,i} +  b_{k,i}  +e_{k,i} \Big).
\eeqan 
The regret is thus bounded by
\beqa\label{eq:regret.bound.1}
\wt {\text{Regret}}(K)\leq \square \sum_{k,h} (\epsilon_{k,h} + \bar \epsilon_{k,h} + b_{k,h} +  e_{k,h}).
\eeqa 
We now bound those 4 terms. It is easy to check that  $\sum_{k,h} \epsilon_{k,h}$ and $\sum_{k,h} \bar\epsilon_{k,h}$ are sums of martingale differences, which are bounded using Azuma's inequality, and lead to a regret of $\wt O(H\sqrt{T})$ without dependence on the size of state and action space. The leading terms in the regret bound comes from the sum of the exploration bonuses $\sum_{k,h}  b_{k,h}$ and the estimation errors $\sum_{k,h} e_{k,h}$.

\paragraph{Step 2: Bounding the martingales $\sum_{k,h} \epsilon_{k,h}$ and $\sum_{k,h} \bar\epsilon_{k,h}$.}
Using Azuma's inequality we deduce
\beqa \label{eq:bound.epsilon} 
\sum_{k,h} \epsilon_{k,h} \stackrel{(Az)}{\leq} \square H\sqrt{TL},\qquad 
\sum_{k,h} \bar \epsilon_{k,h} \stackrel{(Az)}{\leq} \sqrt{\square TL}.
\eeqa

\paragraph{Step 3: Bounding the exploration bonuses $\sum_{k,h}b_{k,h}$:}
Using the pigeon-hole principle, we have
\beqa
\sum_{k,h}b_{k,h} =  \square HL \sum_{k,h} \sqrt{\frac{1}{n_{k,h}}} 
= \square HL \sum_{x,a} \sum_{n=1}^{N_{K}(x,a)} \sqrt{\frac{1}{n}} 
\leq \square HL\sqrt{SAT}.\label{eq:bound.b}
\eeqa

\paragraph{Step 4: Bounding on the estimation errors $\sum_{k,h} e_{k,h}$.} 
Using CH, w.h.p.~we have
$ e_{k,h} = (\wh P^{\pi_k}_{k} -P^{\pi_k})  V^{*}_{h+1} \stackrel{(CH)}{\leq} \square H \sqrt{\frac{L}{n_{k,h}}}$.
Thus this bound on the estimation errors are of the same order as the exploration bonuses (which is the reason we choose those bonuses...). 

\paragraph{Putting everything together:}
Plugging Eq.~\ref{eq:bound.epsilon} and Eq.~\ref{eq:bound.b} into Eq.~\ref{eq:regret.bound.1} (and adding the smaller order term) we deduce
\beqan 
{\text{Regret}}(K)\leq \wt {\text{Regret}}(K)\! \leq \!\square \big(H^{\frac32}L\sqrt{SAK} + H^2 S^2AL^2\big).
\eeqan 
\subsection{Sketch Proof of Theorem 2}
The proof of Theorem 1 relied on proving by a straightforward induction over $h$ that $\Omega=\{V_{k,h}\geq V^*_h, \forall k,h\}$ hold with high probability. In the case of  exploration bonuses defined by:

\begin{align}
 b_{k,h}(x,a) = \underbrace{\sqrt{\frac{\square L \V_{Y\sim \wh P^{\pi_k}_k(\cdot|x,a)}\big(V_{k,h+1}(Y)\big)}{N_{k}(x,a)}} +  \frac{\square HL}{N_{k}(x,a)}}_{\mbox{empirical Bernstein}}+ \underbrace{\sqrt{\frac{\min\left(\square H^3S^2AL^2 \sum_{y} \frac{\wh P_k(y|x,a)}{N'_{k,h+1}(y)},H^2 \right)}{N_k(x,a)}}}_{\mbox{additional bonus}} \label{eq:def.b},
\end{align}
the backward induction over $h$ is not straightforward. Indeed, if the $V_{k,h+1}$ are upper bounds on $V^*_{h+1}$, it is not necessarily the case that the empirical variance of $V_{k,h+1}$ are upper bound on the empirical variance of $V^*_{h+1}$. However we can prove by (backward) induction over $h$ that $V_{k,h+1}$ is sufficiently close to $V^*_{h+1}$ to guarantee that the variance of those terms are sufficiently close to each other so that the additional bonus (additional bonus in Eq.~\ref{eq:def.b}) will make sure that $V_{k,h}$ is still an upper-bound on $V^*_h$. More precisely, define the set of indices:
\beqan
 [k,h]_{hist}& \eqdef \{ (i,j), s.t. (1\leq i\leq k \wedge 1\leq j\leq H)  
 \vee (i=k \wedge h < j\leq H)\},
\eeqan
and the event $\Omega_{k,h} \eqdef \{V_{i,j}\geq V^*_h, (i,j)\in [k,h]_{hist} \}$. Our induction is the following: 
\begin{itemize}
 \item Assume that $\Omega_{k,h}$ holds. Then we prove that $(V_{k,h+1}-V^*_{h+1})(y)\leq \square H\sqrt{\frac{SAL}{N'_{k,h+1}(y)}}$.
 \item We deduce that {\small $\V_{Y\sim \wh P_k(\cdot|x,a)}\big(V_{k,h+1}(Y)\big) + \square H^3S^2AL^2 \sum_{y} \frac{\wh P_k(y|x,a)}{N'_{k,h+1}(y)} \! \geq\! \V_{Y\sim \wh P_k(\cdot|x,a)}\big(V^*_{h+1}(Y)\big)$}, so the additional bonus compensates for the possible variance difference. Thus $V_{k,h}\geq V^*_h$ and $\Omega_{k,h-1}$ holds.
\end{itemize}
So in order  to prove that all values computed by the algorithm are upper bounding $V^*$, we just need to prove that under  $\Omega_{k,h}$, we have  $(V_{k,h+1}-V^*_{h+1})(y)\leq \min(\square H^{1.5}SL\sqrt{\frac{A}{N'_{k,h+1}(y)}},H)$, which is obtained by deriving the following regret bound on
\beqa
\wt R_{k,h}(y) \eqdef \sum_{i\leq k} (V_{i,h+1} - V^{\pi_i}_{h+1})(x_{i,h+1}) \I\{x_{i,h+1}=y \}
\leq \square HL\sqrt{SA N'_{k,h+1}(y)}.  \label{eq:regret.x}
\eeqa
Indeed, since $\{ V_{i,h} \}_i$ is a decreasing sequence in $i$, we have
\beqan
(V_{k,h+1}-V^*_{h+1})(y)\leq \wt R_{k,h+1}(y) / N'_{k,h+1}(y)\leq \square HL\sqrt{SA / N'_{k,h+1}(y)}.
\eeqan
Once we have proven that w.h.p., all computed values are upper bounds on $V^*$ (i.e.~event $\Omega$), then we prove that under $\Omega$, the following regret bound holds:
\beq\label{eq:full.regret}
{\text{Regret}}(K)\leq \wt {\text{Regret}}(K) \leq \square (HL\sqrt{SAK} + H^2 S^2 A L^2).
\eeq

The proof of Eq.~\ref{eq:regret.x} relies on the same derivations as those used for proving Eq.~\ref{eq:full.regret}. The only two differences  being that (i) $HK$ is replaced by $N'_{k,h+1}(y)$, the number of times a state $y$ was reached at time $h+1$, up to episode $k$, and (ii) the additional $\sqrt{H}$ factor which comes from the fact that at any episode, $N'_{k,h+1}(y)$ can only tick once, whereas the total number of transitions from $y$ during any episode can be as large as $H$. The full proof of Eq.~\ref{eq:regret.x} will be given in details in the appendix. We now give a proof sketch of Eq.~\ref{eq:full.regret} under $\Omega$.

Similar steps used for proving Theorem 1 apply. The main difference compared to Theorem 1 is the bound on the sum of the exploration bonuses and the estimation errors (which we consider in Steps 3' and 4' below). This is where we can remove the $\sqrt{H}$ factor. The use of the Bernstein inequality makes it possible to bound both of those terms in terms of the expected sum of variances (under the current policy $\pi_k$ at any episode $k$) of the next-state values (for that policy), and then using recursively the Law of Total Variance to conclude that this quantity is nothing but the variance of the returns. This step is detailed now. For simplicity of the exposition of this sketch  we neglect second order terms.

\paragraph{Step 3': Bounding the sum of exploration bonuses $b_{k,h}$.}
We have
\begin{align*}
&\sum_{k,h}b_{k,h} =  \square \sqrt{L} \underbrace{\sum_{k,h}\sqrt{\frac{\V_{Y\sim \wh P^{\pi_k}_h(\cdot|x_{k,h})}\big(V_{k,h+1}(Y)\big)}{n_{k,h}}}}_{\mbox{main term}}
+\underbrace{\sqrt{\frac{\min\left(\square H^3S^2AL^2 \sum_{y} \frac{\wh P_k(y|x,a)}{N'_{k,h+1}(y)},H^2\right) }{N_k(x,a)}} + \sum_{k,h}\frac{\square L}{N_{k}(x,a)}}_{\mbox{second order term}}.
\end{align*}
By Cauchy-Schwarz, the main term is bounded by $\Big(\sum_{k,h}\wh\V_{k,h+1} \sum_{k,h}\frac{1}{n_{k,h}}\Big)^{1/2}$, where\\ $\wh\V_{k,h+1}\eqdef \V_{Y\sim \wh P^{\pi_k}_h(\cdot|x_{k,h})}\big(V_{k,h+1}(Y)\big)$. Since $\sum_{k,h}\frac{1}{n_{k,h}} \leq \square SA \ln(T)$ by the pigeon-hole principle, we now focus on the term $\sum_{k,h}\wh\V_{k,h+1}$.

We now prove that $\wh\V_{k,h+1}$  is close to $\V^{\pi_k}_{k,h+1}\eqdef \V_{Y\sim P^{\pi_k}_h(\cdot|x_{k,h})}\big(V^{\pi_k}_{h+1}(Y)\big)$ by bounding the following quantity:
\beqa 
\wh \V_{k,h+1}-\V^{\pi_k}_{k,h+1}
&=& \wh P^{\pi_k} V_{k,h+1}^2 - (\wh P^{\pi_k}V_{k,h+1})^2  - P^{\pi_k}(V^{\pi_k}_{h+1})^2+ (P^{\pi_k}V^{\pi_k}_{h+1})^2 \notag
\\
&\overset{(i)}{\leq}& \wh P^{\pi_k} V_{k,h+1}^2 - P^{\pi_k}(V^{\pi_k}_{h+1})^2+ 2H(P^{\pi_k}-\wh P^{\pi_k})V^*_{h+1}\notag
\\
&\overset{(ii)}{\leq}& \underbrace{ (\wh P^{\pi_k} - P^{\pi_k}) V^2_{k,h+1}}_{(a_{k,h})} + \underbrace{P^{\pi_k} (V_{k,h+1}^2-(V^{\pi_k}_{h+1})^2)}_{(a'_{k,h})} +\square H^2\sqrt{\frac{L}{n_{k,h}}}
,\label{eq:bound.residual.variance}
\eeqa
where $(i)$ holds since under $\Omega_{k,h}$, $V_{k,h}\geq V^*_{h}\geq V^{\pi_k}_{h}$ and $(ii)$ holds due to Chernoff Hoeffding. 

\paragraph{Step 3'-a: bounding $\sum_{k,h} \wh \V_{k,h+1}-\V^{\pi_k}_{k,h+1}$.} Using similar argument as those used in \cite{UCRLAuer}, we have that 
$$a_{k,h}\leq H^2 \|\wh P^{\pi_k} - P^{\pi_k}\|_1\leq \square H^2 \sqrt{SL/n_{k,h}},$$ 
(where $n_{k,h}\eqdef N_k(x_{k,h},\pi_k(x_{k,h}))$). Thus from the pigeon-hole principle, $\sum_{k,h} a_{k,h}\leq \square H^2 S\sqrt{ATL}$.

Now $a'_{k,h}$ is bounded as
\beqan 
a'_{k,h} \leq 2H P^{\pi_k} (V_{k,h}-V^{\pi_k}_h) = 2H  P^{\pi_k} \wh \Delta_{k,h}.
\eeqan 

Thus using Azuma's inequality,
\beqan 
\sum_{k,h} a'_{k,h} \stackrel{(Az)}{\leq} 2H \sum_{k,h} \wh \delta_{k,h+1} + \square H^2\sqrt{TL} 
\leq 2 H^2 U + \square H^2\sqrt{TL},
\eeqan 
where $U$ is defined as an upper-bound on the pseudo regret: $U\eqdef \sum_{k,h} (b_{k,h}+e_{k,h}) + \square H\sqrt{T}$ (an upper bound on the r.h.s.~of Eq.~\ref{eq:regret.bound.1}).

\paragraph{Step 3'-b: bounding $\sum_{k,h} \V^{\pi_k}_{k,h+1}$.} (Dominant term)

For any episode $k$, $\E[\sum_{h} \V^{\pi_k}_{k,h+1}|\H_k]$ is the expected sum of variances of the value function $V^\pi_k(y)$ at the next state $y\sim P^{\pi_k}(\cdot|x_{k,h})$ under the true transition model for the current policy. 
A recursive application of the law of total variance \citep[see e.g.,][]{Munos99cdc,azar2013minimax,corr/LattimoreHutter} shows that this quantity is nothing else than the variance of the return (sum of $H$ rewards) under policy $\pi_k$: $\V\big(\sum_h r(x_{k,h},\pi_k(x_{k,h})) \big)$, which is thus bounded by $H^2$. Finally, using Freedman's (Fr) inequality to bound $\sum_{k,h} \V^{\pi_k}_{k,h+1}$ by its expectation (see the exact derivation in the appendix), we deduce
\beqa
\sum_{k,h} \V^{\pi_k}_{k,h+1} \stackrel{(Fr)}{\leq} \sum_k \E \Big[\sum_{h} \V^{\pi_k}_{k,h+1}|\H_k\Big] + \square H^2 \sqrt{TL}
\leq TH + \square H^2 \sqrt{TL}. \label{eq:LTV}
\eeqa 

Thus, using Eq.~\ref{eq:LTV}, Eq.~\ref{eq:bound.residual.variance} and the bounds on $\sum a_{k,h}$ and $\sum a'_{k,h}$, we deduce that
\beqan 
\sum_{k,h} b_{k,h} \leq \square L\sqrt{(TH + H^2U) SA}.
\eeqan 

\paragraph{Step 4': Bounding the sum of estimation errors $\sum_{k,h} e_{k,h}$.} We now use Bernstein inequality to bound the estimation errors
\beqan
\sum_{k,h} e_{k,h} = \sum_{k,h} (\wh P^{\pi_k}_k - P^{\pi_k}) V^*_{h+1}(x_{k,h})
\leq \sum_{k,h} \square \sqrt{\frac{  \V^*_{k,h+1}  }{n_{k,h}}} + \square\frac{HL}{n_{k,h}},
\eeqan 
where $\V^*_{k,h+1}\eqdef \V_{Y\sim P^{\pi_k}(\cdot|x_{k,h})}\big(V^*_{h+1}(Y)\big)$. Now, in a very similar way as in Step 3' above, we relate $\V^*_{k,h+1}$ to $\V^{\pi_k}_{k,h+1}$ and use the Law of total variance to bound $\sum_{k,h} \V^{\pi_k}_{k,h+1} $ by $HT$ and deduce that 
\beqan 
\sum_{k,h} e_{k,h} \leq \square L\sqrt{(TH + H^2U) SA}.
\eeqan
From Eq.~\ref{eq:regret.bound.1} we see that $U \leq \square L \sqrt{(TH + H^2U) SA}$ thus $U\leq \square (L \sqrt{HSAT} + H^2SAL^2)$. This implies Eq.~\ref{eq:full.regret}.

So the reason we are able to remove the $\sqrt{H}$ factor from the regret bound comes from the fact that the sum, over $H$ steps, of the variances of the next state values (which define the amplitude of the confidence intervals) is at most bounded by the variance of the return. Intuitively this means that the size of the confidence intervals do not add up linearly over $H$ steps but grows as $\sqrt{H}$ only. Although the sequence of estimation errors are not independent over time, we are able to demonstrate a concentration of measure phenomenon that shows that those estimation errors concentrate as if they were independent.

\section{Conclusion}

In this paper we refine the familiar concept of optimism in the face of uncertainty.
Our key contribution is the design and analysis of the algorithm \mbox{\uBF,} which addresses two key shortcomings in existing algorithms for optimistic exploration in finite MDPs.
First we apply a concentration to the value as a whole, rather than the transition estimates, this leads to a reduction from $S$ to $\sqrt{S}$.  
Next we apply a recursive law of total variance to couple estimates across an episode, rather than at each time step individually, this leads to a reduction from $H$ to $\sqrt{H}$.

Theorem \ref{thm: ucbvi_2 regret} provides the first regret bounds which, for sufficiently large $T$, match the lower bounds for the problem $\wt{O}(\sqrt{HSAT})$ up to logarithmic factors. It remains an open problem whether we can match the lower bound using this approach for  small $T$.
We believe that the higher order term can be improved from  $\wt O(H^2S^2A)$ to $\wt O(HS^2A)$ by a more careful analysis, i.e., a more extensive use of Freedman-Bernstein inequalities.
The same applies to the term of order $H\sqrt{T}$ which can be improved  to $\sqrt{HT}$.  

These results are particularly significant because they help to estabilish the information-theoretic lower bound of reinforcement learning at  $\Omega(\sqrt{HSAT})$ \cite{osband2016lower}, whereas it was suggested in some previous work that lower-bound should be of $\Omega(H\sqrt{SAT})$.
Moving from this big-picture insight to an analytically rigorous bound is non-trivial.
Although we push many of the technical details to the appendix, our paper also makes several contributions in terms of analytical tools that may be useful in subsequent work. In particular we believe that  the way we construct the exploration bonus and confidence intervals in \uCH is novel to the literature of RL. Also the constructive approach in the proof of \uCH\hspace{-0.2cm}, which bootstraps the regret bounds to prove that $V_{k,h}$s are ucbs, is another analytical contribution of this paper. 

\section*{Acknowledgements}
The authors would like to thank Marc Bellemare and all the other wonderful colleagues at DeepMind for many hours of discussion and insight leading to this research.
We are also grateful for the anonymous reviewers for their helpful comments and for fixing several mistakes in an earlier version of this paper.
 
\newpage

\bibliography{Refs}
\bibliographystyle{icml2017} 
\onecolumn
\begin{appendices}

We begin by introducing some notation in Sect.~\ref{sect:notation} and Sect.~\ref{sect:table.notation}. We then provide the full analysis of UCBVI in Sect.~\ref{sect:main.analysis}.

\section{Table of Notation}
\label{sect:table.notation}

\begin{table}[H]
\centering
\begin{tabular}{l c l }
\hline
\textbf{Symbol} & & \textbf{Explanation}
\\
\hline
$\S$& & The state space
\\
$\A$& & The action space
\\
$\pi_k$& & The policy at episode $k$ 
\\
$P$& &The transition distribution
\\
$R$& &The reward function
\\
$S$& & Size of state space
\\
$A$& & Size of action space
\\
$H$& & The horizon length
\\
$T$ and $T_k$& & The total number of steps and number of steps up to episode $k$
\\
$K$& & The total number of episodes
\\
$N_k(x,a)$& &Number of visits to state-action pair $(x,a)$ up to episode $k$
\\
$V^*_h$& &Optimal value function $V^*$
\\
$\T$& & Bellman operator
\\
$V_{k,h}$& & The estimate of value function at step $h$ of episode $k$
\\
$Q_{k,h}$& & The estimate of action value function at step $h$ of episode $k$
\\
$b$& &The exploration bonus
\\
$L$& & $\ln(5SAT/\delta)$
\\
$N_k(x,a,y)$& & Number of transitions from $x$ to $y$ upon taking action $a$ up to episode $k$
\\
$N'_{k,h}(x,a)$& &Number of visits to state-action pair  $(x,a)$ at step $h$ up to episode $k$
\\
$N'_{k,h}(x)$& &Number of visits to state $x$ at step $h$ up to episode $k$
\\
$\wh P_k(y|x,a)$& & The empirical transition distribution from $x$ to $y$ upon taking action $a$ up to episode $k$ 
\\
$\wh\V_{k,h}(x,a)$& & The empirical next-state variance of $V_{k,h}$ for every $(x,a)$
\\
$\V^*_h(x,a)$& & The next-state variance of $V^*$ for every state-action pair $(x,a)$
\\
$\wh \V^*_{k,h}(x,a)$&& The empirical  next-state variance of $V^*_h$ for every state-action pair $(x,a)$ at episode $k$
\\
$\V^{\pi}_h(x,a)$& & The next-state variance of $V^{\pi}_h$ for every state-action pair $(x,a)$
\\
$b'_{i,j}(x)$& & $\min\left(\frac{100^2S^2H^2AL^2}{N'_{i,j}(x)},H^2\right)$
\\
$[(x,a)]_{k}$& & Set of typical state-action pairs
\\
${[k]}_{\text{typ}}$ & & Set of typical episodes
\\
${[y]}_{k,x,a}$ & & Set of typical next states at every episode $k$ for every $(x,a)$ 
\\
$\mathrm{Regret}(k)$ & & The regret after $k$ episodes
\\
$\wt{\mathrm{Regret}}(k)$ & & The upper-bound regret after $k$ episodes
\\
$\mathrm{Regret}(k,x,h)$ & & The regret upon encountering state $x$ at step $h$ after $k$ episodes
\\
$\wt{\mathrm{Regret}}(k,x,h)$ & & The regret upon encountering state $x$ at step $h$ after $k$ episodes
\\
$\Delta_{k,h}$& & One step regret at step $h$ of episode $k$
\\
$\wt \Delta_{k,h}$& & One step upper-bound regret at step $h$ of episode $k$
\\
$\wt \Delta_{\text{typ},k,h}$& & One step upper-bound regret at step $h$ of episode $k$ for typical episodes
\\
$\Mc_t$ & & The martingale operator
\\
$\eps_{k,h}$ and $\bar\eps_{k,h}$& & Martingale difference terms 
\\
$c_1(v,n)$, $c_2(p,n)$ and $c_3(n)$& & The confidence intervals for the value function and transition distribution 
\\
$C_{k,h}$& & Sum of confidence intervals $c_1$ up to step $h$ of episode $k$
\\
$B_{k,h}$& & Sum of exploration bonuses $b$ up to step $h$ of episode $k$
\\
$\calE$& & The high probability event under which the concentration inequality holds
\\
$\Omega$& & The high probability event under which the estimates $V_{k,h}$ are ucbs
\\
$\H_t$& & The history of all random events up to time step $t$
\\
\hline
\end{tabular}
\end{table}

\section{Notation}
\label{sect:notation}

Let denote the total number of times that we visit state $x$ while taking action $a$ at step $h$ of all episodes up to episode $k$  by $N'_{k,h}(x,a)$. 
We also use the notation $N'_{k,h}(x)=\sum_{a\in\A} N'_{k,h}(x,a)$ for the total number of visits to state $x$ at time step $h$ up to episode $k$.   Also define the empirical next-state variance $\wh\V_{k,h}(x,a)$,  the next-state variance of optimal value function $\V^*_{h}(x,a)$ the next-state empirical variance of optimal value function $\wh \V^*_{k,h}(x,a)$ and the next-state variance of $V^{\pi}_h$ as 

\beqan
\wh\V_{k,h}(x,a)&\eqdef& \text{Var}_{y\sim \wh P_k(\cdot|x,a)}(V_{k,h+1}(y)),
\\
\V^*_h(x,a) &\eqdef&\text{Var}_{y\sim P(\cdot|x,a)} (V^*_h(y)),
\\
\wh \V^*_{k,h}(x,a) &\eqdef&\text{Var}_{y\sim \wh P_{k}(\cdot|x,a)} (V^*_h(y)),
\\
\V^{\pi}_{h}(x,a) &\eqdef&\text{Var}_{y\sim P(\cdot|x,a)} (V^{\pi}_h(y)).
\eeqan
for every $(x,a)\in\S\times\A$ and $k\in[K]$ and $h\in[H]$. We further introduce some short-hand notation: we use the lower case to denote the functions evaluated at the current state-action pair, e.g., we write $n_{k,h}$ for $N_k(x_{k,h},\pi_k(x_{k,h},h))$ and $v_{k,h}$ for  $V_{k,h}(x_{k,h})$.  Let also denote $\V^*_{k,h}=\V^*_{k,h}(x_{k,h},\pi_k(x_{k,h},h))$ and $\V^{\pi_k}_{k,h}=\V^{\pi_k}_{k,h}(x_{k,h},\pi_k(x_{k,h},h))$ for every $k\in[K]$ and $h\in[H]$.
Also define $b'_{i,j}(x)=\min\left(\frac{100^2S^2H^2AL^2}{N'_{i,j+1}(x)},H^2\right)$ for every $x\in\S$.

\subsection{``Typical'' state-actions and  steps}
  In our analysis we split the episodes into 2 sets:  
  the set of ``typical'' episodes in which the number of visits to the encountered state-actions are  large  and the rest of the episodes. We then prove a tight regret bound for the typical episodes. As the total count of other episodes is bounded this technique provides us with the desired result. The set of typical state-actions pairs for every episode $k$ is defined as follows

\beqan
[(x,a)]_{k}&\eqdef&\{(x,a): (x,a)\in\S\times\A, N_{h}(x,a)\geq  H, N'_{k,h}(x)\geq H \}.
\eeqan

Based on the definition of  ${[(x,a)]}_{\text{typ}}$  we define the set of typical episodes and the set of typical state-dependent episodes as follow

\beqan
{[k]}_{\text{typ}}&\eqdef&\{i:i\in[k],\forall h\in [H],
(x_{i,h},\pi_i(x_{i,h},h))\in{[(x,a)]}_{k},i\geq 250 HS^2AL\},
\\
{[k]}_{\text{typ},x}&\eqdef&\{i:i\in[k],\forall h\in [H],
(x_{i,h},\pi_i(x_{i,h},h))\in{[(x,a)]}_{k},N'_{k,h}(x)\geq 250 HS^2AL\}.
\eeqan

Also for every  $(x,a)\in\S\times\A$ the set of typical next states at every episode $k$ is defined as follows

\beqan
{[y]}_{k,x,a}&\eqdef&\{y: y\in \S, N_k(x,a)P(y|x,a)\geq 2H^2L\}.
\eeqan

Finally let denote  $[y]_{k,h}=[y]_{k,x_{k,h},\pi_k(x_{x_{k,h}})}$ for every $k\in[K]$ and $h\in[H]$.

\subsection{Surrogate regrets}
Our ultimate goal is to prove bound on the regret $\mathrm{Regret}(k)$.  However in our analysis we mostly focus on bounding the surrogate regrets. Let  $\wt \Delta_{k,h}(x)\eqdef V_{k,h}(x)-V^{\pi_k}_h(x)$ for every $x\in\S$, $h\in[H]$ and  $k\in [K]$. Then the upper-bound regret $\wt{\mathrm{Regret}}$ defined as follows 

\beqan
\wt{\mathrm{Regret}}(k)&\eqdef& \sum_{i=1}^k \wt \delta_{i,1}.
\eeqan

$\wt{\mathrm{Regret}}(k)$ is useful in our  analysis since it provides an upperbound on the true regret  $\mathrm{Regret}(k)$. So one can bound $\wt{\mathrm{Regret}}(k)$  as a surrogate for $\mathrm{Regret}(k)$. 

We also define the corresponding per state-step regret and upper-bound regret for every state $x\in\X$ and step $h\in[H]$, respectively, as follows

\beqan
\text{Regret}(k,x,h) &\eqdef& \sum_{i=1}^k\mathbb{I}(x_{i,h}=x)\delta_{i,h},
\\
\wt{\mathrm{Regret}}(k,x,h) &\eqdef& \sum_{i=1}^k\mathbb{I}(x_{i,h}=x)\wt \delta_{i,h}. 
\eeqan

\subsection{Martingale difference sequences}

In our analysis we rely heavily on the theory of martingale  sequences to prove bound on the regret incurred due to encountering a random sequence of states. We now provide some definitions and notation in that regard.

We  define the following martingale operator for every $k\in[K]$, $h\in [H]$ and $F:\S\to\Re$. Also let  $t=(k-1)H+h$ denote the time stamp at step $h$ of episode $k$ then

\beqan
\Mc_{t} F &\eqdef& P^{\pi_k}_h F-F(x_{k,h+1}).
\eeqan

Let $\H_{t}$ be the history  of all random events up to (and including) step $h$ of episode $k$ then we have that $\E(\Mc_t F|\H_t) = 0 $. Thus $\Mc_{t} F$ is  a martingale difference w.r.t. $\H_{t}$. Also  let $G$ be a real-value function depends on $H_{t+s}$ for some integer $s>0$. Then we generalize our definition of operator $\Mc_{t}$ as

\beqan
\Mc_{t} G &\eqdef& \E \left(\left.G(\H_{t+s})\right|\H_{t}\right)-G(\H
_{t+s}),
\eeqan
where  $\E$ is over the randomization of the sequence of states  generated by the sequence of policies $\pi_{k},\pi_{k+1},\dots$.  Here also $\Mc_{t} G$ is a martingale difference w.r.t. $\H_{t}$.

Let define $\Delta_{\text{typ},k,h}:\S\to \Re$  as follows for every $k\in[K]$ and $h\in[H]$ and $y\in\S$

\beqan
\wt \Delta_{\text{typ},k,h+1}(y)&\eqdef&\sqrt{\frac{\I_{k,h}(y)}{n_{k,h}p_{k,h}(y)}}\wt \Delta_{k,h+1}(y),
\eeqan
where the function $p_{k,h}:\S\to[0,1]$ is defined as  $p_{k,h}(y)=P^{\pi_k}_h(y|x_{k,h})$  and $\I_{k,h}(y)$ writes for $\I_{k,h}(y)=\I(y\in [y]_{k,h})$ for every $y\in\X$. We also define the following martingale differences which we use frequently

\beqan
\eps_{k,h}&\eqdef& \Mc_{t} \wt \Delta_{k,h+1},
\\
\bar\eps_{k,h} &\eqdef&\Mc_{t} \wt\Delta_{\text{typ},k,h+1}.
\eeqan

\subsection{ High probability events}
\vspace{0.5cm}

 We now introduce the high probability events $\Ec$  and $\Omega_{k,h}$ under which the regret is small.
 
 Let use the shorthand notation  $L\eqdef\ln\left(\frac{5 SAT}\delta\right)$. Also for every $v>0$, $p\in[0,1]$ and $n>0$ let define the confidence intervals  $c_1$,  $c_2$ and $c_3$, respectively, as follow

\beqan
c_1(v,n)&\eqdef&2\sqrt{ \frac{vL}n}+\frac{14HL}{3n},
\\
c_2(p,n)&\eqdef&2\sqrt{\frac{p(1-p)L} n}+\frac{2L}{3n},
\\
c_3(n)&\eqdef&2\sqrt{\frac{SL}n}.
\eeqan

Let $\calP$ be the set of all probability distributions on $\S$. Define the following confidence set for every $k=1,\dots,K$, $n>0$ and $(x,a)\in \S\times\A$

\beqan
\calP(k,h,n,x,a,y) &\eqdef& \Big\{\wt P(\cdot|x,a) \in \calP : |(\wt P - P) V^*_h(x,a)| \leq \min\left( c_1\left(\V^*_h(x,a) ,n\right), c_1\left(\wh \V^*_{k,h}(x,a) ,n\right)\right)
\\
& &|\wt P(y|x,a)-P(y|x,a)|\leq c_2\left( P(y|x,a), n \right),
\\
& &\|\wt P(\cdot|x,a)-P(\cdot|x,a)\|_1\leq c_3(n)\Big\}.
\eeqan

We now define the random event $\Ec_{\wh P}$ as follows

\beqan
\Ec_{\wh P}&\eqdef& \Big\{\wh P_k(y|x,a)\in \calP(k, h, N_k(x,a), x, a, y),\forall k\in [K], \forall h \in [H],   \forall (y,x,a)\in \S\times\S\times\A \Big\}.
\eeqan

Let $t$ be a positive integer. Let $\Fc=\{f_s\}_{s\in[t]}$ be a set of real-value functions on $\H_{t+s}$, for some integer $s>0$.  We now define the following random events for every $\bar w>0$ and $\bar u>0$ and $\bar c>0$: 

\beqan
\Ec_{\text{az}}(\F, \bar u, \bar c)&\eqdef&\left\{\sum_{s=1}^t \Mc_{s} f_{s} \leq 2\sqrt{t \bar u^2 \bar c}  \right\},
\\
\Ec_{\text{fr}}(\F,\bar w, \bar u, \bar c)&\eqdef&\left\{\sum_{s=1}^t  \Mc_s f_s\leq 4\sqrt{ \bar w c}+\frac{14\bar u\bar c}3  \right\}.
\eeqan

We also use the short-hand notation $\Ec_{\text{az}}(\F, \bar u)$  and $\Ec_{\text{fr}}(\F,\bar w, \bar u)$ for $\Ec_{\text{az}}(\F, \bar u, L)$  and $\Ec_{\text{fr}}(\F,\bar w, \bar u, L)$, respectively.

Now let define the following sets of random variables for every $k\in[K]$ and  $h\in[H]$:

\beqan
\F_{\wt \Delta,k,h}&\eqdef&\left\{\wt \Delta_{i,j}: i\in[k] , h<j\in[H-1]\right\},
\\
\F'_{\wt \Delta,k,h}&\eqdef&\left\{\wt \Delta_{\text{typ},i,j}: i\in[k] , h<j\in[H]\right\},
\\
\F_{\wt \Delta,k,h,x}&\eqdef&\left\{\wt \Delta_{i,j}\I(x_{i,h}=x): i\in[k] , h<j\in[H]\right\},
\\
\F'_{\wt \Delta,k,h,x}&\eqdef&\left\{\wt \Delta_{\text{typ},i,j}\I(x_{i,h}=x): i\in[k] , h<j\in[H]\right\},
\\
\Gc_{\V,k,h}&\eqdef&\left\{\sum_{j=h+1}^{H}\V^{\pi_i}_{j}: i\in[k] , h<j\in[H]\right\},
\\
\Gc_{\V,k,h,x}&\eqdef&\left\{\sum_{j=h+1}^{H}\V^{\pi_i}_{j}\I(x_{i,h}=x): i\in[k] , h<j\in[H]\right\},
\\
\F_{b',k,h}&\eqdef&\left\{b'_{i,j}: i\in[k] , h<j\in[H-1]\right\},
\\
\F_{b',k,h,x}&\eqdef&\left\{ b'_{i,j}\I(x_{i,h}=x): i\in[k] , h<j\in[H]\right\}.
\eeqan

We now define the  high probability event $\Ec$ as follows 

\beqan
\Ec&\eqdef& \Ec_{\wh P} \bigcap\bigcap_{\substack{k\in[K]\\h\in[H]\\x\in\S}}\bigg[ \Ec_{\text{az}}(\F_{\wt \Delta,k,h},H)
\bigcap\Ec_{\text{az}}(\F'_{\wt \Delta,k,h},1/\sqrt{L})
\bigcap\Ec_{\text{az}}(\F_{\wt \Delta,k,h,x},H)
\bigcap\Ec_{\text{az}}(\F'_{\wt \Delta,k,x,h},1/\sqrt{L})
\\
& &\bigcap\Ec_{\text{fr}}(\G_{\V,k,h},H^4T,H^3)
\bigcap\Ec_{\text{az}}(\G_{\V,k,h,x}, H^5N'_{k,h}(x),H^3)
\bigcap\Ec_{\text{az}}(\F_{b',k,h},H^2 )
\bigcap\Ec_{\text{az}}(\F_{b',k,h,x},H^2)\bigg].
\eeqan

The following lemma shows that the event $\Ec$ holds with high probability:

\begin{lemma}
Let $\delta>0$ be a real scalar. Then the event $\Ec$ holds w.p. at least $1-\delta$.
\end{lemma}

\begin{proof}
To prove this result we need to show that a set of concentration inequalities with regard to the empirical model $\wh P_k$ holds simultaneously. For every $h\in[H]$ the Bernstein inequality  combined with a  union bound argument, to take into account that $N_k(x,a)\in[T]$ is a random number, leads to the following inequality w.p. $1-\delta$ \citep[see, e.g.,][for the statement of the Bernstein inequality and the application of the union bound in  similar cases, respectively.]{CBLu06:book,Bubeck2012regretBandit}

\beqa
\label{eq:Bernstein.V}
\left| \left[(P - \wh P_k) V^*_h\right] (x,a) \right| &\leq& \sqrt{ \frac{2 \V^*_h (x,a) \ln  \left(\frac{2T}{\delta}\right)  }{N_{k}(x,a)} } +\frac{2H\ln\left(\frac{2T}\delta\right)}{3N_{k}(x,a)},
\eeqa
where we rely on the fact that $V^*_h$ is uniformly bounded by  $H$. Using the  same argument but this time with the Empirical Bernstein inequality \citep[see, e.g.,][]{maurer2009empirical}, for $N_{k}(x,a)>1$, leads to

\beqa
\label{eq:Bernstein.Vhat}
\left| \left[(P - \wh P_k) V^*_h\right] (x,a) \right| &\leq& \sqrt{ \frac{2 \wh \V^*_{k,h} (x,a) \ln \left(\frac {2T}\delta\right) }{N_{k}(x,a)}}+\frac{7H\ln\left(\frac{2T}{\delta}\right)}{3N_{k}(x,a)}.
\eeqa

The Bernstein inequality combined with a union bound argument on $N_{k}(x,a)$ also implies the following bound  w.p. $1-\delta$

\beqan
\left|N_k(y,x,a) - N_{k}(x,a) P(y|x,a)\right| \leq \sqrt{2N_k(x,a)\text{Var}_{z\sim P(\cdot|x,a)}({\bf 1}(z=y))\ln\left(\frac {2T}\delta\right)}+ \frac{2\ln\left(\frac{2T}\delta\right)}{3},
\eeqan
which implies the following bound w.p. $1-\delta$:

\beqa
\label{eq:Bernstein.transit}
\left|\wh P_{k}(y|x,a) -  P(y|x,a)\right| &\leq& \sqrt{\frac{P(y|x,a)(1-P(y|x,a))\ln\left(\frac {2T}{\delta}\right)}{N_{k}(x,a)}}+ \frac{2\ln\left(\frac{2T}{\delta}\right)}{3N_{k}(x,a)}.
\eeqa

A similar result  holds on $\ell_1$-normed estimation error of the transition distribution. The result of  \citep{weissman2003inequalities} combined with a union bound on $N_k(x,a)\in[T]$ implies w.p. $1-\delta$

\beqa
\label{eq:transition.l1}
\left\|\wh P_k(\cdot|x,a) - P(\cdot|x,a)\right\|_1 &\leq& \sqrt{\frac{ 2S\ln \left(\frac{2T}\delta\right)}{N_k(x,a)}}.
\eeqa

We now focus on bounding the sequence of martingales. Let $n>0$ be an integer and $u,\delta>0$ be some real scalars. Let the sequence of random variables $\{X_1,X_2,\dots,X_n\}$  be a sequence of martingale differences w.r.t. to some filtration $\F_n$. Let this sequence be uniformly bounded from above and below by $u$. Then the Azuma's  inequality \citep[see, e.g.,][]{CBLu06:book} implies that w.p. $1-\delta$

\beqa
\label{eq:azuma.ineq}
\sum_{i=1}^{n}X_i &\leq& \sqrt{2nu\ln\left(\frac 1\delta\right)}.
\eeqa

When the sum of the variances $\sum_{i=1}^n \text{Var}(X_i|\F_i)\leq w$ for some $w>0$  then the following sharper bound due to \citet{freedman1975tail} holds w.p. $1-\delta$

\beqa
\label{eq:freedman.ineq}
\sum_{i=1}^{n}X_i &\leq& \sqrt{2w\ln\left(\frac 1\delta\right)}+\frac{2u\ln\left(\frac 1\delta\right)}{3}.
\eeqa

Let $k\in[K]$, $h\in[H]$ and $x\in\X$. Then the inequality of Eq.~\ref{eq:azuma.ineq} immediately implies that the  following events holds w.p. $1-\delta$:

\beqa
& &\Ec_{\text{az}}\left(\F_{\wt \Delta,k,h},H,\ln\left(1/\delta\right)\right) \label{eq:mart1},
\\
& &
\Ec_{\text{az}}\left(\F'_{\wt \Delta,k,h},1/\sqrt{L},\ln\left(1/\delta\right)\right), \label{eq:mart2}
\\
& &
\Ec_{\text{az}}\left(\F_{b',k,h},H^2 , \ln\left(1/\delta\right)\right)\label{eq:mart3}.
\eeqa

Also Eq.~\ref{eq:azuma.ineq} combined with a union bound argument over all $N'_{k,h}(x)\in[ T]$ \citep[see, e.g.,][for the full description of the application of union bound argument in the case of martigale process with random stopping time]{bubeck2011xarmed} implies that the following  events hold w.p. $1-\delta$

\beqa
& &\Ec_{\text{az}}\left(\F_{\wt \Delta,k,h,x},H,\ln\left(T/\delta\right)\right) \label{eq:mart4},
\\
& &
\Ec_{\text{az}}\left(\F'_{\wt \Delta,k,h,x},1/\sqrt{L},\ln\left(T/\delta\right)\right), \label{eq:mart5}
\\
& &\Ec_{\text{az}}\left(\F_{b',k,h,x},H^2 , \ln\left(T/\delta\right)\right)\label{eq:mart6}.
\eeqa

Similarly the inequality of Eq.~\ref{eq:freedman.ineq} leads to the following events hold w.p. $1-\delta$

\beqa
& &\Ec_{\text{fr}}\left(\G_{\V,k,h},\bar w_{k,h}, ,H^3,\ln\left(T/\delta\right)\right),\label{eq:freed1}
\\
& &\Ec_{\text{fr}}\left(\G_{\V,k,h,x},\bar w_{k,h,x}, ,H^3,\ln\left(1/\delta\right)\right),\label{eq:freed2}
\eeqa
where $\bar w_{k,h}$ and $\bar w_{k,h,x}$ are upper bounds on $W_{k,h}$ and $W_{k,h,x}$, respectively, defined as

 \beqa
 W_{k,h}&\eqdef& \sum_{i=1}^k \text{Var}\left(\left.\sum_{j=h}^{H-1}\V^{\pi}_{i,j+1}\right|\H_{i,1}\right), \label{eq:variance.sum.bound1}
 \\
 W_{k,h,x}&\eqdef& \sum_{i=1}^k \I(x_{i,h}=x)\E\left(\left.\sum_{j=h}^{H-1}\V^{\pi}_{i,j+1}\right|\H_{i,1}\right). \label{eq:variance.sum.bound2}
 \eeqa

So to establish a value for $\bar w_{k,h}$ and $\bar w_{k,h,x}$ we need to prove bound  on $W_{k,h}$ and  $W_{k,h,x}$. Here we only prove this  bound for $W_{k,h}$ as the  proof techniques to bound $W_{k,h,x}$ is identical to the way we bound $W_{k,h}$.

\beqa
\label{eq:variance.sum.bound}
W_{k,h}\leq \sum_{i=1}^k \E\left(\left.\sum_{j=h}^{H-1}\V^{\pi_k}_{i,j+1}\right|\H_k\right)^2
\leq H^3\sum_{i=1}^k \E\left(\left.\sum_{j=h}^{H-1}\V^{\pi_k}_{i,j+1}\right|\H_k\right).
\eeqa

 Now let the sequence  $\{x_1,x_2,\dots,x_H\}$ be the sequence of states encountered by following some policy $\pi$ throughout an episode $k$. Then the recursive application of LTV leads to \citep[see e.g.,][for the proof.]{Munos99cdc,corr/LattimoreHutter}

\beqa
\label{eq:LTV.value}
\E\left(\sum_{j=h}^{H-1}\V^{\pi}(x_j,\pi(x_{j},j))\right)&=&\mathrm{Var}\left(\sum_{j=h}^{H-1} r^{\pi}(x_{j})\right).
\eeqa

By combining  Eq.~\ref{eq:LTV.value} into  Eq.~\ref{eq:variance.sum.bound} we deduce

\beqa
\label{eq:bound.Wkh}
W_{k,h}&\leq& H^3\sum_{i=1}^k \text{Var}\left(\left.\sum_{j=h}^{H-1}r_{k,h}\right|\H_k\right)\leq H^5k=H^4T_k.
\eeqa

Similarly the following bound holds on $W_{k,h,x}$ 

\beqa
\label{eq:bound.Wkhx}
W_{k,h,x}&\leq&  H^5 N_{k,h}(x).
\eeqa

Plugging the bounds of Eq.~\ref{eq:bound.Wkh} and Eq.~\ref{eq:bound.Wkhx} in to the bounds of Eq.~\ref{eq:freed1} and Eq.~\ref{eq:freed2} and a union bound over all $N_{k,h}(x)\in[T]$ leads to the following events hold  w.p. $1-\delta$:

\beqa
& &\Ec_{\text{fr}}\left(\G_{\V,k,h},  H^4T ,H^3,\ln\left(1/\delta\right)\right),\label{eq:freed1plugged}
\\
& &\Ec_{\text{fr}}\left(\G_{\V,k,h,x}, H^5 N_{k,h}(x) ,H^3,\ln\left(T/\delta\right)\right).\label{eq:freed2plugged}
\eeqa

Combining the results of Eq.~\ref{eq:Bernstein.V}, Eq.~\ref{eq:Bernstein.Vhat}, Eq.~\ref{eq:Bernstein.transit}, Eq.~\ref{eq:transition.l1},  Eq.~\ref{eq:mart1}, Eq.~\ref{eq:mart2} Eq.~\ref{eq:mart3}, Eq.~\ref{eq:mart4}, Eq.~\ref{eq:mart5},  Eq.~\ref{eq:mart6}, Eq.~\ref{eq:freed1plugged} and  Eq.~\ref{eq:freed2plugged}   and taking a union bound over these random events as well as all possible $k\in[K]$, $h\in[H]$ and  $(s,a)\in\S\times\A$  proves the result.

\end{proof}

\subsubsection{UCB Events}

Let $k\in[K]$ and $h\in[H]$. Denote the set of  steps for which the value functions are obtained before $V_{k,h}$ as 

\beqan
[k,h]_{\text{hist}}=\{(i,j): i\in [K], j\in[H], i<k \vee (i=k \wedge j>h) \}.
\eeqan

Let $\Omega_{k,h}=\{V_{i,j}\geq V^*_h, \forall (i,j)\in[k,h]_{\text{hist}} \}$ be the event under which $V_{i,j}$ prior to  $V_{k,h}$ computation are upper bounds on the optimal value functions. Using backward induction on $h$ (and standard concentration inequalities) we will prove that $\Omega_{k,h}$ holds under the event $\Ec$ (see Lem.~\ref{lem:bound.UCB.VI2. highprob}).

\subsection{Other useful notation}

  Here we define some other notation that we use throughout the proof. We denote the total count of steps up to episode $k\in[K]$ by $T_k\eqdef H(k-1)$. We first define  $c_{4,k,h}$, for every $h\in[H]$ and $k\in[K]$, as follow

\beqan
c_{4,k,h}=\frac{4H^2SAL}{n_{k,h}}.
\eeqan

for every $k\in[K]$ , $h\in[H]$ and $x\in[x]$  we also introduce the following notation which we use later when we sum up the regret: 

\beqan
C_{k,h}&\eqdef&\sum_{i=1}^k\I(i\in [k]_{\text{typ}}) \sum_{j=h}^{H-1} c_{1,i,j},
\\
B_{k,h}&\eqdef&\sum_{i=1}^k\I(i\in [k]_{\text{typ}}) \sum_{j=h}^{H-1} b_{i,j},
\\
C_{k,h,x}&\eqdef&\sum_{i=1}^k\I(i\in [k]_{\text{typ},x},x_{k,h}=x) \sum_{j=h}^{H-1} c_{1,i,j},
\\
B_{k,h,x}&\eqdef&\sum_{i=1}^k\I(i\in [k]_{\text{typ},x},x_{k,h}=x) \sum_{j=h}^{H-1} b_{i, j},
\eeqan
where $c_{1,k,h}$ is the shorthand-notation for $c_1(v^*_{k,h},n_{k,h})$. We also define the upper bound $U_{k,h}$ and $U_{k,h,x}$  for every $k\in[K]$ , $h\in [H] $  and $x\in\S$ as follows, respectively

\beqan
U_{k,h}&\eqdef&e\sum_{i=1}^{k}\sum_{j=h}^{H-1}\left[b_{i,j}+ c_{1,i,j}+c_{4,i,j}\right]+(H+1)\sqrt{T_kL},
\\
U_{k,h,x}&\eqdef&e\sum_{i=1}^{k}\sum_{j=h}^{H-1}\left[b_{i,j}+ c_{1,i,j}+c_{4,i,j}\right]+(H+1)^{3/2}\sqrt{N'_{k,h}(x)L},
\eeqan

\section{Proof of the Regret Bounds}
\label{sect:main.analysis}

Before we start the main analysis we state the following useful lemma that will be used  frequently in the  analysis:

\begin{lemma}
\label{lem:var.diff.bound}
let $X\in\R$ and $Y\in\R$ be two random variables. Then following bound holds for their variances

\beqan
\text{Var}(X)\leq 2 \left[\text{Var}(Y)+\text{Var}(X-Y) \right].
\eeqan

\end{lemma}

\begin{proof}

The following sequence of inequalities hold

\beqan
\text{Var}(X)=\E(X-Y-\E(X-Y)+Y-\E(Y))^2\leq 2\E(X-Y-\E(X-Y))^2+2\E(Y-\E(Y))^2.
\eeqan

The result  follows from the definition of variance.
\end{proof}

We proceed by proving the following key lemma which shows that proves bound on $\Delta_{k,h}$ under the assumption that $V_{k,h}$ is UCB w.r.t.  $V^*_h$. 

\begin{lemma}
\label{lem:1step.upperbound.VI}
Let $k\in[K]$ and $h\in[H]$. Let the events $\Ec$ and $\Omega_{k,h}$ hold. Then the following bound holds on $\delta_{k,h}$ and $\wt \delta_{k,h}$:

\beq
\delta_{k,h} \leq \wt \delta_{k,h}\leq e\sum_{i=h}^{H-1}\left[\eps_{k,i} + 2 \sqrt{L}\bar \eps_{k,i} +c_{1,k,i}+b_{k,i}+ c_{4,k,i}\right].
\eeq

\end{lemma}

\begin{proof}
For the ease of exposition we abuse the notation and drop the dependencies on $k$, e.g., we write $x_1$, $\pi$ and $V_1$  for $x_{k,1}$, $\pi_k$ and $V_{k,1}$, respectively. We proceed by bounding $\wt \delta_{h}$ under the event $\Ec$  at every step $0<h<H$:

\beqa 
\wt \delta_h&=& 
\T_h V_{h+1}(x_h)-\T^\pi_{h} V^{\pi}_{h+1}(x_h) \notag
\\
&=&  [\wh P^\pi_h V_{h+1}](x_h) +b_h- [P^\pi_h V^\pi_{h+1}](x_h)\notag\\
&=& b_h + [(\wh P^\pi_h- P^\pi_h)V^*_{h+1}](x_h)
+[(\wh P^\pi_h- P^\pi_h)(V_{h+1}-V^*_{h+1})](x_h)+ [P^\pi_h (V_{h+1}-V^\pi_{h+1})](x_h)\notag
\\
&\leq&\wt \delta_{h+1}+\eps_h+b_h+c_{1,h}+\underbrace{[(\wh P^\pi_h- P^\pi_h) (V_{h+1}-V^*_{h+1})](x_h)}_{(a)},
\eeqa \label{eq:step.bound1.VI2}
where the last inequality follows from the fact that under the event $\Ec$ we have that $[(\wh P^\pi_h- P^\pi_h)V^*_{h+1}](x_h)\leq c_{1,h}$. We now bound $(a)$:

\beqan
\label{eq:step.bound.HOT1.VI2}
(a) &=& \sum_{y\in \S}(\wh P^\pi_h(y|x_h)-P^\pi_h(y|x_h)) (V_{h+1}(y)-V^*_{h+1}(y))
\\
&\overset{(I)}{\leq}& \sum_{y\in \S}\left[2\sqrt{\frac{p_h(y)(1-p_h(y))L}{n_h}}+\frac{4 L}{3n_h}\right]  \Delta_{h+1}(y)
\\
&\leq&2\sqrt{L} \underbrace{\sum_{y\in \S}\sqrt{\frac{p_h(y)}{n_h}}\wt \Delta_{h+1}(y)}_{(b)}+\frac{4SHL}{3n_h},
\eeqan
where $(I)$ holds under the event $\Ec$. We proceed by bounding $(b)$: 

\beqa
\label{eq:b.bernstein.prob}
(b) &=& \underbrace{\sum_{y\in [y]_h}\sqrt{\frac{p_h(y)}{n_h}}\wt \Delta_{h+1}(y)}_{(c)}+\underbrace{\sum_{y\notin [y]_h}\sqrt{\frac{p_h(y)}{n_h}}\wt \Delta_{h+1}(y)}_{(d)}.
\eeqa

The term $(c)$ can be bounded as follows

\beqa
(c)&=&\sum_{y\in [y]_h}P^\pi_h(y|x_h)\sqrt{\frac{1}{n_h p_h(y)}}\wt \Delta_{h+1}(y)= \bar\eps_h+ \sqrt{\frac{1}{n_h p_h(x_{h+1})}}\mathbb{I}(x_{h+1}\in[y]_h)\wt \delta_{h+1}\notag
\\
&\leq&\bar\eps_h+ \sqrt{\frac{1}{4LH^2}}\wt \delta_{h+1},\label{eq:c.high.visit.bound}
\eeqa

where in the last line we rely on the definition of $[y]_{h}$. We now bound (d):

\beq
\label{eq:d.low.visit.bound}
(d)=\sum_{y\notin [y]_h}\sqrt{\frac{p_h(y)n_h}{n_h^2}}\wt \Delta_{h+1}(y)\leq\frac{SH\sqrt{4LH^2}}{n_h}.
\eeq

By combining Eq.~\ref{eq:c.high.visit.bound} and Eq.~\ref{eq:d.low.visit.bound} into Eq.~\ref{eq:b.bernstein.prob}  we deduce

\beq
\label{eq:b.bernstein.prob.bound}
(b)\leq \frac{SH\sqrt{4LH^2}}{n_h}+ \sqrt{\frac{1}{4LH^2}}\wt \delta_{h+1}+\bar\eps_h.
\eeq

By combining Eq.~\ref{eq:b.bernstein.prob.bound} and Eq.~\ref{eq:step.bound.HOT1.VI2} into Eq.~\ref{eq:step.bound1.VI2}  we deduce

\beqan
\wt \delta_{h} \leq  \eps_h +  2\sqrt{L}\bar \eps_h + b_h+c_{1,h}+c_{4,h}+\left(1+\frac 1H \right)\wt \delta_{h+1}.
\eeqan

Let denote $\gamma_h=(1+1/H)^{h}$. The previous bound combined with an induction argument implies that

\beqan
\wt \delta_{h}&\leq \sum_{i=h}^{H-1}\gamma_{i-h}\left[\eps_{i} + 2\sqrt{L} \bar \eps_{i} +c_{1,i}+c_{4,i}+b_{i}\right].
\eeqan

The inequality $\ln(1+x)\leq x$ for every $x> -1$ leads to  $\gamma_h \leq \gamma_H \leq  e$ for every $h\in [H]$. This combined with the assumption that $v_h\geq v^*_h$ under the event $\Omega_h$ completes the proof. 
\end{proof}

\begin{lemma}
\label{lem:sum.Delta.bound}
Let $k\in[k]$ and $h\in[H]$. Let the events $\Ec$ and $\Omega_{k,h}$ hold. Then

\beqan
\sum_{i=1}^{k-1}\delta_{i,h} \leq \sum_{i=1}^{k-1}\wt \delta_{i,h} \leq e \sum_{i=1}^{k-1}\sum_{j=h}^{H-1}\left[\eps_{i,j} + 2\sqrt{L}\bar \eps_{i,j} +b_{i,j}+ c_{1,i,j}+c_{4,i,j}\right].
\eeqan
\end{lemma}

\begin{proof}
The proof follows by summing up the bounds of Lem.~\ref{lem:1step.upperbound.VI} and taking into acoount the fact if $\Omega_{k,h}$ holds then $\Omega_{i,j}$ for all $(i,j)\in[k,h]_{hist}$ hold.
\end{proof}

To simplify the bound of Lem.~\ref{lem:sum.Delta.bound} we prove bound  on sum of the martingales $\eps_{k,h}$ and $\bar \eps_{k,h}$

\begin{lemma}
\label{lem:azuma.eps}
Let $k\in[k]$ and $h\in[H]$. Let the events $\Ec$ and $\Omega_{k,h}$ hold. Then the following bound holds

\beqa
\sum_{i=1}^{k}\sum_{j=h}^{H-1}\eps_{i,j}&\leq& H\sqrt{(H-h)kL} \leq H\sqrt{T_k L},
\\
\sum_{i=1}^{k}\sum_{j=h}^{H-1}\bar\eps_{i,j}&\leq&\sqrt{(H-h)k}\leq \sqrt{T_k}.
\eeqa

Also the following bounds holds for every $x\in\X$ and $h\in\H$:

\beqa
\sum_{i=1}^{k}\I(x_{i,h}=x)\sum_{j=h}^{H-1}\eps_{i,j}  &\leq& H\sqrt{(H-h)N'_{k,h}(x)L}  , \label{eq:azuma.eps_k.x}
\\
\sum_{i=1}^{k}\I(x_{i,h}=x)\sum_{j=h}^{H-1}\bar\eps_{i,j}  &\leq&\sqrt{(H-h)N'_{k,h}(x)}.
\label{eq:azuma.bareps_k.x}
\eeqa

\end{lemma}

\begin{proof}
 The fact that the event $\Ec$  holds implies that the events $\Ec_{\text{az}}(\F_{\wt \Delta,k,h},H)$, $\Ec_{\text{az}}(\F'_{\wt \Delta,k,h},\frac{1}{\sqrt{L}})$  , $\Ec_{\text{az}}(\F_{\wt \Delta, k, h,x},H)$ and $\Ec_{\text{az}}(\F'_{\wt \Delta,  x, k,h}, \frac{1}{\sqrt{L}})$ hold. Under these events the inequalities of the statement hold. This combined with the fact that $(H-h)k\leq T_k$ completes the proof.

\end{proof}

We now bound the sum of $\delta$s in terms of the upper-bound $U$:

\begin{lemma}
\label{lem:bound.sum.delta}
Let $k\in[K]$ and $h\in[H]$. Let  the events $\Ec$  and $\Omega_{k,h}$ holds. Then the following bounds hold for every $h\in[H]$ $x\in\S$

\beqan
\sum_{i=1}^{k} \delta_{i,h} &\leq&  \sum_{i=1}^{k} \wt \delta_{i,h}\leq U_{k,h} \leq U_{k,1},
\\
\sum_{i=1}^{k} \I(x_{i,h}=x)\delta_{i,h} &\leq&  \sum_{i=1}^{k}\I(x_{i,h}=x) \wt \delta_{i,h}\leq U_{k,h,x}. \leq U_{k,1,x}.
\eeqan
\end{lemma}

\begin{proof}
The proof follows by incorporating  the result of Lem.~\ref{lem:azuma.eps} into Lem.~\ref{lem:sum.Delta.bound} and taking into account that for every $h\in[H]$ the term $U_{k,h}$ ($U_{k,h,x}$) is a summation of non-negative terms which are also contained  in $U_{k,1}$ ($U_{1,h,x}$).
\end{proof}

\begin{lemma}
\label{lem:bound.total.delta}
Let $k\in[K]$ and $h\in[H]$. Let  the events $\Ec$  and $\Omega_{k,h}$ holds. Then the following bounds hold for every $x\in\S$

\beqan
\sum_{i=1}^{k}\sum_{j=h}^{H}\delta_{i,j} &\leq& \sum_{i=1}^k\sum_{j=h}^H\wt \delta_{i,j} \leq H U_{k,1},
\\
\sum_{i=1}^{k}\I(x_{i,h}=x)\sum_{j=h}^{H}\delta_{i,j} &\leq& \sum_{i=1}^k\I(x_{i,h}=x)\sum_{j=h}^H\wt \delta_{i,j} \leq H U_{k,1,x}.
\eeqan

\end{lemma}

\begin{proof}
The proof follows by summing up the bounds of Lem.~\ref{lem:bound.sum.delta}.
\end{proof}

We now focus on bounding the  terms  $C_{k,h}$  ($C_{k,h,x}$) and  $B_{k,h}$  ($B_{k,h,x}$) in Lem.~\ref{lem:bound.est.err} and Lem.~\ref{lem:BT.bound}, respectively. Before we proceed with the proof of Lem.~\ref{lem:bound.est.err} and Lem.~\ref{lem:BT.bound}. we prove the following key result which bounds sum of the variances of $V^{\pi}_{k,h}$ using an LTV argument:

\begin{lemma}
\label{lem:var.LTV.bound}
Let $k\in [K]$ and $h\in[H]$. Then under the events $\Ec$ and  $\Omega_{k,h}$ the following hold for every $x\in\S$   

\beqan
\sum_{i=1}^{k} \sum_{j=h}^{H-1} \V^{\pi}_{i,j+1}&\leq& T_kH+2\sqrt{H^4T_kL}+\frac {4H^3L}3,
\\
\sum_{i=1}^{k} \I(x_{i,h}=x)\sum_{j=h}^{H-1} \V^{\pi}_{i,j+1}&\leq& N'_{k,h}(x)H^2+2\sqrt{H^5N'_{k,h}(x)L}+\frac {4H^3L}3.
\eeqan

\end{lemma}

\begin{proof}

Under $\Ec$ the  events $\Ec_{\text{fr}}(\Gc_{\V,k,h}, H^4T_k, H^3)$ and $\Ec_{\text{fr}}(\Gc_{\V,k,h,x}, H^5 N_{k,h}(x) , H^3)$  hold which then imply:

\beqa
\sum_{i=1}^{k} \sum_{j=h}^{H-1} \V^{\pi}_{i,j+1}
&\leq&\sum_{i=1}^k\E\left(\left.\sum_{j=h}^{H-1}\V^{\pi}_{i,j+1}\right|\Hc_{k,h}\right)+2\sqrt{H^4T_kL}+\frac {4H^3L}3,\label{eq:expected.var.total}
\\
\sum_{i=1}^{k} \I(x_{i,h}=x) \sum_{j=h}^{H-1} \V^{\pi}_{i,j+1}
&\leq&\sum_{i=1}^k \I(x_{i,h}=x)\E\left(\left.\sum_{j=h}^{H-1}\V^{\pi}_{i,j+1}\right|\Hc_{k,h}\right)+2\sqrt{H^5N'_{k,h}L}+\frac {4H^3L}3.\label{eq:expected.var.x}
\eeqa

The LTV argument of Eq.~\ref{eq:LTV.value} then leads to

\beqa
\sum_{i=1}^k\E\left(\left.\sum_{j=h}^{H-1}\V^{\pi}_{i,j+1}\right|\Hc_{i,h}\right)&=&\sum_{i=1}^k \mathrm{Var}\left(\sum_{j=h+1}^H r^{\pi}_{i,j}\right)\leq KH^2=TH \label{eq:bound.total.variance},
\\
\sum_{i=1}^k\I(x_{i,h}=x)\E\left(\left.\sum_{j=h}^{H-1}\V^{\pi}_{i,j+1}\right|\Hc_{i,h}\right)&=&\sum_{i=1}^k \I(x_{i,h}=x)\mathrm{Var}\left(\sum_{j=h+1}^H r^{\pi}_{i,j}\right)\leq N'_{k,h}(x)H^2.
\label{eq:bound.total.variance.x}
\eeqa

Eq.~\ref{eq:expected.var.total} and Eq.~\ref{eq:expected.var.x} combined with  Eq.~\ref{eq:bound.total.variance} and Eq.~\ref{eq:bound.total.variance.x}, respectively, complete the proof.

\end{proof}

\begin{lemma}
\label{lem:var.value.diff.bound}
Let $k\in [K]$ and $h\in[H]$. Then under the events $\Ec$ and  $\Omega_{k,h}$ the following hold for every $x\in\S$

\beqa
\sum_{i=1}^{k}  \sum_{j=h}^{H-1} \left(\V^*_{i,j+1}-\V^{\pi}_{i,j+1}\right)&\leq& 2H^2U_{k,h}+ 4H^2\sqrt{T_kL},\label{eq:var.diff.bound}
\\
\sum_{i=1}^{k} \I(x_{i,h}=x) \sum_{j=h}^{H-1} \left(\V^*_{i,j+1}-\V^{\pi}_{i,j+1}\right)&\leq& 2H^2 U_{k,h,x}+ 4H^{2}\sqrt{HN'_{k,h}(x,a)L}.\label{eq:var.diff.bound.x}
\eeqa

\end{lemma}

\begin{proof}
 We begin by the following sequence of inequalities:

\beqa
\sum_{i=1}^{k} \sum_{j=h}^{H-1} \V^*_{i,j+1}-\V^{\pi}_{i,j+1}&\overset{(I)}{\leq}&\sum_{i=1}^k\sum_{j=h}^{H-1}\E_{y\sim p_{i,j}}\left[\left(V^*_{i,j+1}(y)\right)^2-\left(V^{\pi_i}_{i,j+1}(y)\right)^2\right]\notag
\\
&=&\sum_{i=1}^k\sum_{j=h}^{H-1}\E_{y\sim p_{i,j}}\left[(V^*_{j+1}(y)-V^{\pi_i}_{j+1}(y)(V^*_{j+1}(y)+V^{\pi_i}_{j+1}(y))\right]\notag
\\
&\leq&2H\underbrace{\sum_{i=1}^k\sum_{j=h}^{H-1}\E_{y\sim p_{i,j}}\left(V^*_{j+1}(y)-V^{\pi_i}_{j+1}(y)\right)}_{(a)},\label{eq:bound.d1st}
\eeqa
where  $(I)$ is obtained from the definition of the variance as well as the fact that $V^*_{i,j}\geq V^{\pi}_{k,h}$. The last line also follows from the fact  that $V^{\pi_k}\leq V^*_h\leq H$.

Using an identical argument we can also prove  the following bound for state-dependent difference:

\beqa
\sum_{i=1}^{k} \I(x_{i,h}=x)\sum_{j=h}^{H-1} \V^*_{i,j+1}-\V^{\pi}_{i,j+1}
&\leq&2H\underbrace{\sum_{i=1}^k\I(x_{i,h}=x)\sum_{j=h}^{H-1}\E_{y\sim p_{i,j}}\left(V^*_{j+1}(y)-V^{\pi_i}_{j+1}(y)\right)}_{(b)},\label{eq:bound.d1st.x}
\eeqa

To bound $(a)$ we use the fact that under the event $ \Ec$ the event  $\Ec_{\text{az}}(\F_{\wt \Delta,k,h}, H)$ also holds. This  combined with the fact that under the event $\Omega_{k,h}$ the inequality $\delta_{k,h}\leq \wt \delta_{k,h}$ holds implies that

\beqa
(a)&\leq&\sum_{i=1}^k\sum_{j=h}^{H-1}\wt \delta_{i,j+1}+2H\sqrt{T_kL}\notag
\\
&\leq&H U_{1,h}+2H\sqrt{T_kL},
\label{eq:var.bound.azuma}
\eeqa
where in the last line we rely on the result of Lem.~\ref{lem:bound.total.delta}. Similarly we can prove the following bound for $(b)$ under the events $\Omega_{k,h}$ and $\Ec_{\text{az}}(\F_{\wt \Delta,k,h,x}, H)$:

\beqa
(b)&\leq&\sum_{i=1}^k \I(x_{i,h}=x)\sum_{j=h}^{H-1}\wt \Delta_{i,j+1}+2H^{1.5}\sqrt{N'_{k,h}(x)L}\notag
\\
&\leq&H U_{k,h,x}+2H^{1.5}\sqrt{N'_{k,h}(x)L}.
\label{eq:var.bound.azuma.x}
\eeqa

The result then follows by incorporating the results of Eq.~\ref{eq:var.bound.azuma} and Eq.~\ref{eq:var.bound.azuma.x} into Eq.~\ref{eq:bound.d1st} and Eq.~\ref{eq:bound.d1st.x}, respectively.

\end{proof}

\begin{lemma}
\label{lem:var.vh.diff.bound}
Let $k\in [K]$ and $h\in[H]$. Then under the events $\Ec$ and  $\Omega_{k,h}$ the following hold for every $x\in\S$  

\beqa
\sum_{i=1}^{k} \sum_{j=h}^{H-1} \wh\V_{i,j+1}-\V^{\pi}_{i,j+1}&\leq& 2H^2 U_{k,1}+ 15H^2S\sqrt{AT_kL},\label{eq:empric.diff.var}
\\
\sum_{i=1}^{k} \I(x_{i,h}=x)\sum_{j=h}^{H-1} \wh\V_{i,j+1}-\V^{\pi}_{i,j+1}&\leq& 2H^2U_{k,h,x}+ 15H^2S\sqrt{HAN'_{k,h}(x)L}.
\label{eq:empric.diff.var.x}
\eeqa

\end{lemma}

\begin{proof}

Here we only prove the bound on Eq.~\ref{eq:empric.diff.var}. The proof for the bound of  Eq.~\ref{eq:empric.diff.var.x} can be done in a very similar manner, as it is shown in the previous lemmas (the only difference is that $HN'_{k,h}(x)$ and $U_{k,h,x}$ replace $T_k$ and $U_{k,1}$, respectively).  The following sequence of inequalities hold:

\beqa
\sum_{i=1}^{k} \sum_{j=h}^{H-1} \wh \V_{i,j+1}-\V^{\pi}_{i,j+1}&\overset{(I)}{\leq}&\sum_{i=1}^k \sum_{j=h}^{H-1}\E_{y\sim \wh p_{i,j}}\left(V_{i,j+1}(y)\right)^2-\E_{y\sim p_{i,j}}\left(V^{\pi_i}_{j+1}(y)\right)^2\notag
\\
& &+\sum_{i=1}^k \sum_{j=h}^{H-1}\left(\E_{y\sim p_{i,j}}V^*_{j+1}(y)\right)^2-\left(\E_{y\sim \wh p_{i,j}}V^*_{j+1}(y)\right)^2\notag
\\
& \overset{(II)}{\leq}&\underbrace{\sum_{i=1}^k \sum_{j=h}^{H-1}\E_{y\sim \wh p_{i,j}}\left(V_{i,j+1}(y)\right)^2-\sum_{j=1}^{H-1}\E_{y\sim  p_{i,j}}\left(V_{i,j+1}(y)\right)^2}_{(a)}\notag
\\
& &+\underbrace{\sum_{i=1}^k\sum_{j=h}^{H-1}\E_{y\sim  p_{i,j}}\left[\left(V_{i,j+1}(y)\right)^2-\left(V^{\pi_i}_{j+1}(y)\right)^2\right]}_{(b)}\notag
\\
& &+\underbrace{\sum_{i=1}^k\sum_{j=h}^{H-1}4H^2\sqrt{\frac{L}{n_{k,h}}}}_{(c)}
,\label{eq:bound.var.diff}
\eeqa
where  $(I)$ holds due to the fact that under $\Omega_{k,h}$, $V_{i,j}\geq V^*_{j}\geq V^{\pi_i}_{j}$ and $(II)$ holds under the event  $\calE$.

We now bound $(a)$:

\beqan
(a)&\overset{(I)}{\leq}&\sum_{i=1}^k \sum_{j=h}^{H-1}2H^2\sqrt{\frac{SL}{n_{k,h}}}
\\
&\overset{(II)}{\leq}&3H^2S\sqrt{AT_kL},
\eeqan
where  $(I)$ holds under the event $\Ec$ and $(II)$ holds due to the pigeon-hole argument \citep[see, e.g.,][for the proof]{UCRLAuer}.

Using an identical analysis to the one in  Lem.~\ref{lem:var.vh.diff.bound}  and taking into account that $V_{i,j}\geq V^*_j$  under the event $\Omega_{k,h}$ and $\Ec$ we can bound $(b)$  

\beqan
(b)\overset{(I)}{\leq} 2H\left(\sum_{i=1}^k\sum_{j=h}^{H-1} \wt \delta_{i,j+1}+2H\sqrt{T_kL})\right)&\leq& 2H^2\left(U_{k,1}+2\sqrt{T_kL} \right),
\eeqan
where $(I)$ holds since under  the event $\Ec$ the event $\Ec_{az}(\F_{\wt \Delta,k,h},H)$ holds. Another application of pigeon-hole principle leads to a bound of $6H^2\sqrt{SAT_kL}$ on $(c)$. We then combine this with the bounds on  $(a)$ and $(b)$ to bound Eq.~\ref{eq:bound.var.diff}, which proves the result.

\end{proof}

We now bound $C_{k,h}$ and $C_{k,h,x}$:

\begin{lemma}
\label{lem:bound.est.err}
Let $k\in [K]$ and $h\in[H]$. Then under the events $\Ec$ and  $\Omega_{k,h}$ the following hold for every $x\in\S$  

\beqa
C_{k,h} &\leq& 4\sqrt{H SAT_k}+ 4\sqrt{H^2 U_{k,1}SAL^2},\label{eq:bound.ckh}
\\
C_{k,h,x} &\leq&  4\sqrt{ H^2SA N_{k,h}(x)} + 4\sqrt{H^2 U_{k,h,x}SAL^2}.\label{eq:bound.ckhx}
\eeqa

\end{lemma}

\begin{proof}

Here we only prove the bound on Eq.~\ref{eq:bound.ckh}. The proof for the bound of  Eq.~\ref{eq:bound.ckhx} can be done in a very similar manner, as it is shown in the previous lemmas (the only difference is that $HN'_{k,h}(x)$ and $U_{k,h,x}$ replace $T_k$ and $U_{k,1}$, respectively). The  Cauchy–Schwarz inequality leads to the following sequence of inequalities:

\beqa
C_{k,h} &=&
\sum_{i=1}^k \I (j\in [k]_{\text{typ}})\left(\sum_{j=h}^{H-1} 2\sqrt{\frac{\V^*_{i,j+1}L}{n_{i,j}}} +\frac{4HL}{3n_{i,j}}\right)\notag
\\
&\leq& 2\sqrt{L}\sqrt{\underbrace{\sum_{i=1}^k\sum_{j=h}^{H-1}\V^*_{i,j+1}}_{(a)}}\sqrt{\underbrace{\sum_{i=1}^k\I (i\in [k]_{\text{typ}})\sum_{j=h}^{H-1}\frac{1}{n_{i,j}}}_{(b)}}+\sum_{i=1}^k \I (i\in   [k]_{\text{typ}})\sum_{h=j}^{H-1}\frac{4HL}{3n_{i,j}}\label{eq:xi.bound.ab}
\eeqa

We now prove bounds on $(a)$ and $(b)$ respectively

\beq
\label{eq:bound.a.cd2}
(a)= \underbrace{\sum_{i=1}^k\sum_{j=h}^{H-1}\V^{\pi}_{i,j+1}}_{(c)}+\underbrace{\sum_{i=1}^k\sum_{j=h}^{H-1}\V^*_{i,j+1}-\V^{\pi}_{i,j+1}}_{(d)}.
\eeq

$(c)$  and $(d)$ can be bounded under the events $\Ec$ and $\Omega_{k,h}$ using the results of Lem.~\ref{lem:var.LTV.bound}  and Lem.\ref{lem:var.value.diff.bound}. We then deduce

\beqan
(a) & \leq & HT_k+2H^2 U_{k,1}+6H^2 \sqrt{T_kL}+\frac {4H^2L}3
\\
&\leq& 2 HT_k +2 H^2U_{k,1},
\eeqan
where the last line follows by the fact that for the typical episodes $T_k \geq 250 H^2S^2AL^2$. Thus if $T_k\leq H^2L$  the term $C_{k,h}$ trivially equals to $0$ otherwise the higher order terms are bounded by $O(HT_k)$. 

We now bound $(b)$ using a pigeon-hole argument

\beqan
(b)\leq
2\sum_{(x,a)\in\S\times\A}\sum_{n=1}^{N_k(x,a)}\frac 1n \leq 2SA\sum_{n=1}^{T}\frac 1n\leq 2SA\ln(3T).
\eeqan

Plugging the bound on $(a)$ and $(b)$ into Eq.~\ref{eq:xi.bound.ab} and taking in to account that for the typical episodes $[k]_{\text{typ}}$ we have that $T\geq H^2 L $ completes the proof.

\end{proof}

We now bound $B_{k,h}$:

\begin{lemma}
\label{lem:BT.bound}
Let $k\in [K]$ and $h\in[H]$. Let the bonus is defined according to Algo. 4.  Then under the events $\Ec$ and  $\Omega_{k,h}$ the following hold for every $x\in\S$,

\beqa
B_{k,h}&\leq& 11L\sqrt{T_kHSA}+ 12\sqrt{H^2SAL^2 U_{k,1}}+570H^2S^2AL^2
,\label{eq:Bkh.bound}
\\
B_{k,h,x}&\leq&  11L\sqrt{N'_{k,h}(x)HSA}+ 12\sqrt{H^2SAL^2 U_{k,h,x}}+570H^2S^2AL^2,\label{eq:Bkhx.bound}
\eeqa

\end{lemma}

\begin{proof}
Here we only prove the bound on Eq.~\ref{eq:Bkh.bound}. The proof for the bound of  Eq.~\ref{eq:Bkhx.bound} can be done in a very similar manner, as it is shown in the previous lemmas (the only difference is that $HN'_{k,h}(x)$ and $U_{k,h,x}$ replace $T_k$ and $U_{k,1}$, respectively).
We first notice that the following holds:

\beqan
B_{k,h} \leq \underbrace{\sum_{i=1}^{k} \I (i\in [k]_{\text{typ}})\sum_{j=h}^{H-1}\sqrt{\frac{8\wh\V_{i,j+1}L}{n_{i,j}}}}_{(a)}+ L\underbrace{\sum_{i=1}^{k}\I (i\in [k]_{\text{typ}})\sum_{j=h}^{H-1}\left(\sqrt{\frac{8}{n_{i,j}}\sum_{y\in \S}\wh p_{i,j}(y)\min\left(\frac{100^2S^2H^2AL^2}{N'_{i,j+1}(y)},H^2\right)}\right)}_{(b)}.
\eeqan

 We first note that the bound on $B_{k,h}$ is similar to the bound on $C_{k,h}$. The main difference  (beside the difference in H.O.Ts) is that here $\V^*_{h+1}$ is replaced by $\wh \V_{i,j+1}$. So in our proof we first focus on dealing with this difference.

The  Cauchy–Schwarz inequality leads to:

\beqan
(a) &\leq&\sqrt{8L}\sqrt{\underbrace{\sum_{i=1}^k\sum_{j=h}^{H-1}\wh \V _{i,j+1}(x,a)}_{(c)}}\sqrt{\underbrace{\sum_{i=1}^k\I (k\in [k]_{\text{typ}})\sum_{j=h}^{H-1}\frac{1}{n_{i,j}}}_{(d)}},
\eeqan

The bound on $(d)$ is identical to the corresponding bound in Lem.~\ref{lem:bound.est.err}. So we only focus on bounding $(c)$:

\beq
\label{eq:bound.a.cd}
(c)= \underbrace{\sum_{i=1}^k\sum_{j=h}^{H-1}\V^{\pi}_{i,j+1}}_{(e)}+\underbrace{\sum_{i=1}^k\sum_{j=h}^{H-1}\wh \V_{i,j+1}-\V^{\pi}_{i,j+1}}_{(f)}.
\eeq

$(e)$  and $(f)$ can be bounded in high probability using the results of Lem.~\ref{lem:var.LTV.bound}  and Lem.\ref{lem:var.vh.diff.bound}. This implies

\beqan
(c) & \leq & HT_k+3H^2 U_{k,1}+15H^2S \sqrt{AT_kL}+\frac {4H^2L}3
\\
&\leq&
2HT_k +3 H^2U_{k,1},
\eeqan
where the last line follows by the fact that for the typical episodes $T_k \geq 250 H^2S^2AL$. Thus if $T_k\leq 250 H^2S^2L$ then  $B_{k,h}$ trivially equals to $0$ otherwise the higher order terms are bounded by $O(HT)$.  Combining the bound  on $(b)$ and $(c)$ leads to the following bound on $(a)$:

\beqan
(a)&\leq& 8L\sqrt{ HSAT_k}+ 12 HL\sqrt{SAU_{k,1}}.
\eeqan

To bound $(b)$ we make use of Cauchy-Schwarz inequality again. 

\beqan
(b)&\leq&\sqrt{8\underbrace{\sum_{i=1}^k\I (i\in [k]_{\text{typ}})\sum_{j=h}^{H-1}\sum_{y\in \S}\wh p_{i,j}(y)b'_{i,j+1}(y)}_{(g)}\underbrace{\sum_{i=1}^k\sum_{j=h}^{H-1}\frac {\I (i\in [k]_{\text{typ}})}{n_{i,j}}}_{(h)}}.
\eeqan

The term  $(h)$ bounded by $2SAL$ using a pigeon-hole argument (see Lem.~\ref{lem:bound.est.err}). We proceed by bounding $(g)$:

\beqan
(g)&\leq &\underbrace{\sum_{i=1}^k\sum_{j=h}^{H-1} (\wh p_{i,j}-p_{i,j})b'_{i,j+1}}_{(i)}+\underbrace{\sum_{i=1}^k\sum_{j=h}^{H-1}( p_{i,j}b'_{\V}-b'_{i,j+1}(x_{i,j+1}))}_{(j)}+\underbrace{\sum_{i=1}^k\I (i\in [k]_{\text{typ}})\sum_{j=h}^{H-1}b'_{i,j+1}(x_{i,j+1})}_{k}.
\eeqan

Given that the event $\Ec$ holds the term $(i)$ bounded by $2\sqrt{2}H^2S\sqrt{ALT_k}$ by using the pigeon-hole argument.  Under the event $\Ec$ the event $\Ec_{az}(\Fc_{b',k,h},H^2)$ holds. This implies that the term $(j)$ is also bounded by $2H^2\sqrt{T_kL}$ as it is sum of the martingale differences. The term $(k)$ is also bounded by $20000H^3S^3A^3L^3$ using the pigeon-hole argument. Combining all these bounds together leads to the following bound on $(b)$

\beqan
(b)&\leq&\sqrt{32\sqrt{2}H^2S^2\sqrt{T_kAL^3}+32H^2SA\sqrt{T_kL^3}+ 320000S^4H^4A^2L^3}.
\eeqan

Combining this with the bound on $(a)$ and taking into account the fact that we only bound the $B_{k,h}$ for the typical episodes, in which $T_k\geq 250 H^2S^2AL^2$, completes the proof.

\end{proof}

\begin{lemma}
\label{lem:bound.recursive}

Let the bonus is defined according to Algo. 4. Then under the events $\Ec$ and  $\Omega_{K,1}$ the following hold

\beq
 \text{Regret}(K) \leq \wt{ \text{Regret}}(K) \leq U_{K,1} \leq 15L\sqrt{ HSAT}+16HL{\sqrt{SAU_{k,1}}}+ 820H^2S^2A^2L+2H\sqrt{TL}.
\eeq

\end{lemma}

\begin{proof}
We first notice that $\text{Regret}(K)$ and $\text{Regret}(K)$ are bounded by $U_{k,1}$ due to Lem.\ref{lem:bound.sum.delta}. To bound $U_{k,1}$ we sum up the regret due to $B_{k,h}$ and $C_{k,h}$ from Lem.~\ref{lem:bound.est.err} and Lem.~\ref{lem:BT.bound}. We also bound the sum  $\sum_{k=1}^K\sum_{h=1}^H c_{4,k,h}$ by  $2HSAL$ using a pigeon hole argument.  We also note that $B_{k,h}$ and $C_{k,h}$  only account for the regret of  typical episodes in which $T\geq H^2S^2A^2L$. The regret of those episodes which do not belong to the typical set $[k]_{\text{typ}}$,  can be bounded by  $ O(H^2S^2A^2L^2)$, trivially. 
\end{proof}

The following lemma establishes an explicit bound on the regret:

\begin{lemma}
\label{lem:regret.event}
Let the bonus is defined according to Algo. 4. Then under the events $\Ec$ and  $\Omega_{K,1}$ the following hold

\beq
 \text{Regret}(K) \leq \wt{ \text{Regret}}(K) \leq U_{K,1} \leq 30L\sqrt{ HSAT}+ 2500H^2S^2A L^2+4H\sqrt{TL}.
\eeq

\end{lemma}

\begin{proof}
The proof follows by solving the bound  of Lem.~\ref{lem:bound.recursive} in terms of $U_{k,1}$. which only contributes to the additional regret of $O(H^2L^{2}SA)$.
\end{proof}

\begin{lemma}
\label{lem:bound.basic}

Let the bonus is defined according to Algo. 3. Then under the events $\Ec$ and  $\Omega_{K,1}$ the following holds

\beq
 \text{Regret}(K) \leq \wt{ \text{Regret}}(K) \leq U_{K,1} \leq 20HL\sqrt{SAT}+ 250H^2S^2AL^2.
\eeq

\end{lemma}

\begin{proof}
The proof up to Lem.~\ref{lem:bound.est.err} is identical to the proof  of Lem.~\ref{lem:regret.event}. The main difference is to prove bound on $C_{k,h}$ and $B_{k,h}$ here we use a loose  bound  of  $O(H\sqrt{\frac{SAL}{n_{k,h}}})$ for  both exploration bonus $b_{k,h}$ and the confidence interval $c_{1,k,h}$ and then sum these terms using a pigeon-hole argument \citep[The proof is provided in][]{UCRLAuer} which leads to a bound of $O(H\sqrt{SATL})$ on both $B_{K,1}$ and $C_{K,1}$. Plugging these results into the bound of Lem.~\ref{lem:bound.total.delta}  combined with  the regret of non-typical episodes complete the proof 

\end{proof}

\begin{lemma}
\label{lem:state-step.bound}
Let the bonus is defined according to Algo. 4. Let $k\in [K]$ and $h\in[H]$. Then under the events $\Ec$ and  $\Omega_{k,h}$ the following hold for every $x\in\S$,

\beqan
\text{Regret}(k,x,h) &\leq& \wt{\text{Regret}}(k,x,h)\leq 30 HL\sqrt{ SAN'_{k,h}(x)}+2500H^2S^2AL^2+ 4H^{1.5}\sqrt{N'_{k,h}(x)L}
\\
&\leq& 100 H^{1.5}SL\sqrt{AN'_{k,h}(s)}.
\eeqan
\end{lemma}

\begin{proof}
The proof is similar to the proof of total regret. Here also  we use Lem.~\ref{lem:BT.bound}, Lem.~\ref{lem:bound.est.err} and a pigeon-hole argument  to bound  the regrets due to $B_{k,h}$, $C_{k,h}$  and $c_{4,k,h}$. We then incorporate these terms into Lem.\ref{lem:bound.sum.delta} to bound the regret in terms of $U_{k,h,x}$. The result follows by solving the bound w.r.t. the upper bound $U_{k,h,x}$.
\end{proof}

\begin{lemma}
\label{lem:state-step.bound.error}
Let the bonus  $b$ is defined according to Algo. 4. Let $k\in [K]$ and $h\in[H]$. Then under the events $\Ec$ and  $\Omega_{k,h}$ the following hold for every $x\in\S$

\beqan
V_{k,h}(x)-V^*_h(x)&\leq& 100\sqrt{\frac{H^3S^2AL^2}{N'_{k,h}(s)}}.
\eeqan
\end{lemma}

\begin{proof}
From Lem.~\ref{lem:state-step.bound} we have that

\beqan
& & 100 H^{1.5}SL\sqrt{AN'_{k,h}(s)}
\\
&\geq&\sum_{i=1}^{k}\mathbb{I}(x_{i,h}=x)(V_{i,h}(x)-V^{\pi_i}_h(x))
\\
&\geq&(V_{k,h}(x)-V^*_h(x))\sum_{i=1}^{k}\mathbb{I}(x_{i,h}=x)=N'_{k,h}(x)(V_{k,h}(x)-V^*_h(x)),
\eeqan
where the last inequality holds due to the fact that $V_{k,h}$ by definition is monotonically non-increasing in $k$. The proof then follows by collecting terms. 

\end{proof}

\begin{lemma}
\label{lem:ucb.basic}
Let the bonus $b$ is defined according to Algo. 3. Then under the event $\Ec$  the set of events $\{\Omega_{k,h}\}_{k\in[K],h\in H}$ hold.
\end{lemma}

\begin{proof}
We prove this result by induction. First we notice that for $h=H$ by definition $V_{k,h}=V^*_h$ thus the inequality $V_{k,h}\geq V^*_h$ trivially holds. Thus to prove this result for $h<H$ we only need to show that if the inequality  $V_{k,h}\geq V^*_h$ holds for $h$ it also holds for $h-1$ for every $h < H$:

\beqan
V_{k,h}(x)-V^*_h(x)&=  \T_k V_{k,h-1}(x)-\T V^*(x) \geq b_k(x,\pi^*_h(x))+ \wh P^{\pi^*}_{k,h} V_{k,h+1}(x) - P^{\pi^*}_{h} V_{k,h+1}(x)
\\
&= b_k(x,\pi^*_h(x))+ \wh P^{\pi^*}_{k,h} (V_{k,h+1}-V^*_{h+1})(x) + (\wh P^{\pi^*}_{k,h}-P^{\pi^*}_{h} )V^*_{h+1}(x)
\\
&\geq b_k(x,\pi^*_h(x))+ (\wh P^{\pi^*}_{k,h}-P^{\pi^*}_{h} )V^*_{h+1}(x),
\eeqan
where the last line follows by the induction condition  that $V_{k,h+1}\geq V^*_{h+1}$. The fact that the event $\calE$ hols implies that  $( P^{\pi^*}_{h}-\wh P^{\pi^*}_{k,h} )V^*_{h+1}(x)\leq c_1(N_k(x,\pi^*_h(x)))\leq  b_k(x,\pi^*_h(x))$, which completes the proof.

\end{proof}

\begin{lemma}
\label{lem:bound.UCB.VI2. highprob}
Let the bonus $b$ is defined according to Algo. 4. Then under the event $\Ec$ the set of events $\{\Omega_{k,h}\}_{k\in[K],h\in H}$ hold.

\end{lemma}

\begin{proof}
We prove this result by induction.  We first notice that in the case of the first episode $V_{1,h}=H\geq V^*_h$.

 To prove this result by induction  in the case  of $1<k\in[K]$  we  need to show that in the case of $h\in[H-1]$ if  $\Omega_{k,h+1}$ holds  then $\Omega_{k,h}$ also holds.

If $\Omega_{k,h-1}$ holds then $V_{i, j}\geq V^*_j$ for every $(i,j)\in[k,h]_{\text{hist}}$. We can then invoke the result of Lem.~\ref{lem:state-step.bound.error} which implies

\beqan
\label{eq:bound.val.error}
V_{k, h+1}(x)-V^*_{ h+1}(x)&\leq&\frac{100H^{1.5}SL\sqrt{A}}{\sqrt{N'_{k,h+1}(x)}}.
\eeqan

Using this result which guarantees that  $V_{k,h+1}$ is  close to $V^*_{h+1}$ we prove that $V_{k,h}-V^*_{h}\geq 0$, that is the event $\Omega_{k,h}$ holds.

\beqan
V_{k,h}-V^*_h =  \min(V_{k-1,h},\T_{k,h} V_{i,j+1},H) -V^*_h 
\eeqan

If $V_{k-1,h}\leq \T_{k,h} V_{i,j+1}$ the result $V_{k,h}-V^*_{h}= V_{k-1,h}-V^*_{h}\geq0$ holds trivially. Also if $V_{k-1,h}\geq H$ the result trivially holds. So we only need to consider the case that $\T_{k,h} V_{i,j+1}\leq V_{k-1,h}\leq H$  in that case we have w

\beqan
\label{eq:bound.VK.V*}
V_{k,h}(x)-V^*_h(x) &\geq& \T_{k,h} V_{i,j+1}(x) -\T V^*_{h+1}(x)
\\
&\overset{(I)}{\geq}& b_{k,h}(x,\pi^*(x,h))+\wh P^{\pi^*}_h V_{i,j+1}(x) -  P^{\pi*}_h V^*_{h+1}(x) 
\\
&=& b_{k,h}(x,\pi^*(x,h))+( \wh P^{\pi^*}_h- P^{\pi^*}_h )V^*_{h+1}(x) +  \wh P^{\pi*}_h (V_{i,j+1}-V^*_{h+1})(x)
\\
&\overset{(II)}{\geq}& b_{k,h}(x,\pi^*(x,h))+(\wh P^{\pi^*}_h- P^{\pi^*}_h )V^*_{h+1}(x),
\eeqan

where in $(I)$ we rely on the fact that $\pi_{k,h}$ is the greedy policy w.r.t. $V_{k,h}$. Thus

\beqan
b_{k,h}(x,\pi^*(x,h))+\wh P^{\pi^*}_h V_{i,j+1}(x)&\leq& b_{k,h}(x,\pi_k(x,h))+\wh P^{\pi_k}_h V_{i,j+1}(x).
\eeqan

Also $(II)$ follows from the induction assumption. Under the event $\Ec$ we have

\beqan
V_{k,h}-V^*_h  &\geq&  b_{k,h}- c_{1}(\wh \V^*_h,N_k) 
\\
&\geq& \underbrace{\sqrt{\frac{8\wh \V_{k,h}L}{N_k}}-2\sqrt{\frac{ \wh \V^*_h L}{N_k}}}_{(a)}-\frac{14L}{3N_k}
\\
& &+\sqrt{\dfrac{\wh P_k\left[ 8\min\left(\frac{100^2H^3S^2A^2L^2}{N'_{k,h+1}},H^2\right)\right] }{N_k}} +\frac{14L}{3N_k}.
\eeqan

We now prove a lower bound on $(a)$:

\beqan
(a)\geq 
\begin{cases}
-\sqrt{\dfrac{ 4 \wh \V^*_{k,h}-8\wh \V_{k,h} }{N_k}}  &  \wh \V_{k,h} \leq  \V^*,
\\
0 &\text{otherwise}.
\end{cases}
\eeqan

We proceed by bounding $\wh \V^*_{k,h}$ in terms of  $\wh \V_{k,h}$ from above:
\beqan
 \wh \V^*_{k,h}&\overset{(I)}{\leq}& 2\wh \V_{k,h} + 2\text{Var}_{y\sim \wh P_k}( V_{k,h+1}(y)-V^*_{h+1}(y))\leq  2\wh \V_{k,h} + 2\underbrace{\wh P_k(V_{k,h+1}-V^*_{h+1})^2}_{(b)},
\eeqan
where $(I)$ is an application of Lem.~\ref{lem:var.diff.bound}. We now bound  $(b)$. Combining this result with the result of Eq.~\ref{eq:bound.val.error} leads to the following bound on $(a)$

\beqan
(a)\geq 
\begin{cases}
-\sqrt{\dfrac{8\wh P_k \left[\min\left(\frac{100^2H^3S^2AL^2}{N'_{k,h+1}},H^2\right)\right] }{N_k}}  &  \wh \V_{k,h} \leq \wh \V^*,
\\
0 &\text{otherwise},
\end{cases}
\eeqan
where  the last inequality holds under the event $\Ec$. The proof is completed by  plugging $(a)$ and $(b)$ into Eq.~\ref{eq:bound.VK.V*} which proves that $V_{k,h}\geq V^*_h$ thus the event $\Omega_{k,h} holds$.

\end{proof}

\subsection{Proof of Thm. 1}

The result is a direct consequence of  Lem.~\ref{lem:ucb.basic} and Lem.~\ref{lem:bound.basic} and the fact that the high probability event $\Ec$ holds w.p. $1-\delta$.

\subsection{Proof of Thm. 2}
The result is  a direct consequence of Lem.~\ref{lem:bound.UCB.VI2. highprob} and Lem.~\ref{lem:regret.event} and the fact that the high probability event $\Ec$ holds w.p. $1-\delta$.

\end{appendices}

\end{document}

%% file: symbolsdef.tex
\newcommand{\Nat}{\mathbb{N}}
\newcommand{\Ind}{\mathbb{I}}

\newcommand{\Ec}{\mathcal{E}}
\newcommand{\Fc}{\mathcal{F}}
\newcommand{\Gc}{\mathcal{G}}

\newcommand{\Sc}{\mathcal{S}}
\newcommand{\Ac}{\mathcal{A}}
\newcommand{\Mc}{\mathcal{M}}

\newcommand{\Kc}{\mathcal{K}}

\newcommand{\Hc}{\mathcal{H}}
\newcommand{\A}{\mathcal A}
\renewcommand{\S}{\mathcal S}
\newcommand{\X}{\mathcal X}

\newcommand{\G}{\mathcal G}

\newcommand{\calP}{\mathcal P}

\newcommand{\F}{\mathcal F}
\renewcommand{\H}{\mathcal H}
\newcommand{\E}{\mathbb E}

\newcommand{\calE}{\mathcal E}

\newcommand{\V}{\mathbb V}

\newcommand{\I}{\mathbb I}

\newcommand{\R}{\mathcal{R}}

\newcommand{\eps}{\varepsilon}

\newcommand{\uCH}{\texttt{UCBVI-CH }}
\newcommand{\uBF}{\texttt{UCBVI-BF }}

\newcommand{\beq}{\begin{equation}}
\newcommand{\eeq}{\end{equation}}

\newcommand{\beqa}{\begin{eqnarray}}
\newcommand{\eeqa}{\end{eqnarray}}

\newcommand{\beqan}{\begin{eqnarray*}}
\newcommand{\eeqan}{\end{eqnarray*}}

\newcommand{\wh}{\widehat}
\newcommand{\wt}{\widetilde}

\let\R\undefined 
\newcommand{\R}{\mathbb{R}}

\newcommand{\eqdef}{\stackrel{\rm def}{=}}

\newcommand{\cl}[2][ (]{
\ifthenelse{\equal{#1}{ (}}{\left (#2\right)}{}
\ifthenelse{\equal{#1}{[}}{\left[#2\right]}{}
\ifthenelse{\equal{#1}{\{}}{\left\{#2\right\}}{}
}

\newcommand{\inset}[3][C]{
\ifthenelse{\equal{#1}{C}}{{#2}\in\mathcal{#3}}{}
\ifthenelse{\equal{#1}{T}}{{#2}\in{#3}}{}
}

\newcommand{\ESum}[4][C]
{\ifthenelse{\equal{#1}{C}}{\underset{\inset{#3}{#4}}{\sum}#2}{}
 \ifthenelse{\equal{#1}{T}}{\underset{\inset[N]{#3}{#4}}{\sum}#2}{}
\ifthenelse{\equal{#1}{U}}{\overset{#4}{\underset{#3}{\sum}}#2}{}
\ifthenelse{\equal{#1}{X}}{\sideset{}{_{#3}^{#4}}{\sum}#2}{}
\ifthenelse{\equal{#1}{S}}{\sideset{}{_{#3}^{#4}}{\sum}#2}{}
\ifthenelse{\equal{#1}{O}}{{\underset{ (#3,#4)}{\sum}}#2}{}
\ifthenelse{\equal{#1}{I}}{{\underset{#3}{\sum}}#2}{}}

\newcommand{\VF}[2][N]{
\ifthenelse{\equal{#1}{L}}{V^{\pi}_{\lambda} (#2)}{}
\ifthenelse{\equal{#1}{C}}{V{^{\pi} (#2)}}{}
\ifthenelse{\equal{#1}{T}}{V^*_{\bar{\pi}} (#2)}{}
\ifthenelse{\equal{#1}{CO}}{V{^{*} (#2)}}{}
\ifthenelse{\equal{#1}{Mx}}{V^{\pi}_{\infty} (#2)}{}
\ifthenelse{\equal{#1}{Mxo}}{V^{\pi^*}_{\infty} (#2)}{}
\ifthenelse{\equal{#1}{Mn}}{V^{\pi}_{-\infty} (#2)}{}
\ifthenelse{\equal{#1}{Mno}}{V^{\pi^*}_{-\infty} (#2)}{}
}

\newcommand{\Eval}[1][null]{
\ifthenelse{\equal{#1}{null}}{\mathbb{E}}{\mathbb{E}_{#1}}
}

\newcommand{\qv}[1][null]{
\ifthenelse{\equal{#1}{null}}{Q^*}{Q^{#1}}
}

\newcommand{\T}[1][null]{
\ifthenelse{\equal{#1}{null}}{\mathcal{T}}{\mathcal{T}^{#1}}
}

\newtheorem{lemma}{Lemma}
\newtheorem{assumption}{Assumption}

\newtheorem{theorem}{Theorem}

